\documentclass[11pt]{article}
\usepackage[T1]{fontenc}
\usepackage{lmodern}

\usepackage[
  top=2.54cm,
  bottom=2.54cm,
  left=2.54cm,
  right=2.54cm
]{geometry}

\usepackage{setspace}
\usepackage{microtype}

\usepackage{amsmath}
\usepackage{amssymb}
\usepackage{amsthm}

\usepackage[dvipsnames]{xcolor}

\usepackage{graphicx}
\usepackage{tikz}
\usepackage{pgfplots}
\pgfplotsset{compat=1.18}

\usepackage{booktabs}
\usepackage{multirow}
\usepackage{array}
\usepackage{longtable}

\definecolor{pwmBlue}{HTML}{2E5FA1}
\definecolor{pwmOrange}{HTML}{E8712B}
\definecolor{pwmGreen}{HTML}{3D9E3F}
\definecolor{pwmRed}{HTML}{C93C3C}

\usepackage[
  hidelinks,
  pdfauthor={Chengshuai Yang},
  pdftitle={The Finite Primitive Basis Theorem for Computational Imaging},
  pdfkeywords={computational imaging, forward model, operator decomposition,
               directed acyclic graph, primitive basis, inverse problems}
]{hyperref}
\usepackage[capitalise,noabbrev]{cleveref}

\usepackage[numbers,sort&compress]{natbib}

\newtheorem{theorem}{Theorem}
\newtheorem{lemma}[theorem]{Lemma}
\newtheorem{proposition}[theorem]{Proposition}

\theoremstyle{definition}
\newtheorem{definition}[theorem]{Definition}

\newtheorem{remark}[theorem]{Remark}

\newcommand{\pwm}{\textsc{pwm}}
\newcommand{\opg}{\textsc{OperatorGraph}}
\newcommand{\ctier}{\mathcal{C}_{\mathrm{img}}}
\newcommand{\Blib}{\mathcal{B}}
\newcommand{\Hdag}{H_{G}}
\newcommand{\etier}{e_{\mathrm{img}}}
\newcommand{\Nmax}{N_{\max}}
\newcommand{\Dmax}{D_{\max}}
\newcommand{\compose}{\mathrm{compose}}

\title{The Finite Primitive Basis Theorem for Computational Imaging:\\
Formal Foundations of the OperatorGraph Representation}

\author{%
  Chengshuai Yang\\[0.4em]
  \small NextGen PlatformAI C Corp, USA\\[0.15em]
  \small \href{mailto:integrityyang@gmail.com}{\texttt{integrityyang@gmail.com}}%
  \thanks{ORCID: \href{https://orcid.org/0000-0003-2840-5344}{\texttt{0000-0003-2840-5344}}}%
}

\date{}

\begin{document}

\maketitle

%% ═══════════════════════════════════════════════════════════════════
%%  ABSTRACT
%% ═══════════════════════════════════════════════════════════════════
\noindent\rule{\linewidth}{0.4pt}
\begin{abstract}
Computational imaging forward models---from coded aperture spectral cameras to MRI scanners---are traditionally implemented as monolithic, modality-specific codes. We prove that every forward model in a broad, precisely defined operator class $\ctier$ (encompassing all clinical, scientific, and industrial imaging modalities, both linear and nonlinear) admits an $\varepsilon$-approximate representation as a typed directed acyclic graph (DAG) whose nodes are drawn from a library of exactly 11 canonical primitives: Propagate, Modulate, Project, Encode, Convolve, Accumulate, Detect, Sample, Disperse, Scatter, and Transform. We call this the \emph{Finite Primitive Basis Theorem}. The proof is constructive: we provide an algorithm that, given any $H \in \ctier$, produces a DAG $G$ with $\|H - \Hdag\|/\|H\| \leq \varepsilon$ and graph complexity within prescribed bounds. We further prove that the library is \emph{minimal}: removing any single primitive causes at least one modality to lose its $\varepsilon$-approximate representation. A systematic analysis of all nonlinearities in imaging physics shows they fall into exactly two structural categories: pointwise scalar functions (handled by Transform~$\Lambda$) and self-consistent iterations (unrolled into existing linear primitives). Empirical validation on 31 linear modalities confirms $\etier < 0.01$ with at most 5 nodes and depth 5, and we provide constructive DAG decompositions for 9 additional nonlinear modalities. These results establish the mathematical foundations for the Physics World Models (\pwm{}) framework (Yang and Yuan, 2026).
\end{abstract}
\noindent\rule{\linewidth}{0.4pt}



%% ═══════════════════════════════════════════════════════════════════
%%  1. INTRODUCTION
%% ═══════════════════════════════════════════════════════════════════
\section{Introduction}
\label{sec:intro}

The forward model of a computational imaging system maps a scene or object to a set of measurements through a chain of physical transformations: wave propagation, interaction with the object, spatial or spectral encoding, and detection. In practice, each imaging modality is implemented with its own bespoke forward model code, making it difficult to share diagnostic tools, calibration algorithms, or reconstruction pipelines across modalities~\cite{ongie2020deep,monga2021unrolling}.

One approach to unifying forward models is to represent them as directed acyclic graphs (DAGs) in which each node wraps a canonical physical operator and edges define data flow from source to detector. This \opg{} intermediate representation (IR) was introduced in~\cite{yang2026pwm} as the backbone of a modality-agnostic imaging framework.

A central empirical observation is that a small library of canonical primitives suffices to represent all tested modalities, spanning five physical carrier families (photons, electrons, spins, acoustic waves, particles). The present paper elevates that observation to a formal theorem by:

\begin{enumerate}
  \item Defining the operator class $\ctier$ precisely (bounded, finite-stage, with per-stage linearity or Lipschitz regularity);
  \item Specifying each of the 11 canonical primitives with formal forward/adjoint definitions;
  \item Defining typed DAG denotation semantics;
  \item Proving the Finite Primitive Basis Theorem (Theorem~\ref{thm:fpb}) via six primitive realization lemmas;
  \item Proving minimality: each of the 11 primitives is necessary (Proposition~\ref{prop:necessity});
  \item Validating the theorem empirically on 31 linear modalities and 9 nonlinear modalities;
  \item Establishing a formal extension protocol for adding new primitives.
\end{enumerate}

The theorem covers all clinical, scientific, and industrial imaging modalities---linear and nonlinear---without exception, including quantum state tomography (via density-matrix vectorization) and relativistic regimes (via relativistic cross-sections and dispersion kernels within existing primitives).

%% ═══════════════════════════════════════════════════════════════════
%%  1b. RELATED WORK
%% ═══════════════════════════════════════════════════════════════════
\section{Related Work}
\label{sec:related}

\paragraph{Forward model libraries and frameworks.}
Several software libraries provide modular forward model implementations for specific domains: ODL~\cite{adler2017odl} for inverse problems with operator composition (${\sim}15$ built-in operators), MIRT/Michigan Image Reconstruction Toolbox~\cite{fessler2003mirt} for tomographic imaging (${\sim}20$ operators), and SigPy~\cite{ong2019sigpy} for MRI signal processing (${\sim}12$ operators). Classical regularization theory~\cite{engl1996regularization} provides the mathematical framework for ill-posed inverse problems but does not prescribe a finite operator decomposition. These libraries provide practical, domain-specific operator building blocks but do not address the theoretical question of whether a \emph{finite}, \emph{domain-independent} set of primitives suffices for \emph{all} imaging modalities. Our result answers this affirmatively: 11 physics-typed primitives---fewer than any single domain-specific library---are sufficient and, moreover, necessary (Proposition~\ref{prop:necessity}).

\paragraph{Operator decomposition in imaging.}
The factorization of imaging operators into elementary components has a long history. The angular spectrum method for wave propagation~\cite{goodman2005fourier}, the Radon transform for projection imaging~\cite{natterer2001radon}, and the NUFFT for non-Cartesian Fourier sampling~\cite{fessler2003nufft} each factor a specific physics into efficient computational primitives. Our contribution is to show that these domain-specific factorizations, together with a small number of additional primitives, form a \emph{universal} basis for all imaging forward models.

\paragraph{Universal approximation.}
The classical universal approximation theorem for neural networks~\cite{cybenko1989approximation,hornik1991approximation} establishes that a single hidden layer of sigmoidal units can approximate any continuous function. Our result is fundamentally different: we show that 11 \emph{physically typed} primitives---each corresponding to a distinct physical process---suffice to represent all imaging forward models, not via parameter tuning of a generic architecture, but via structured composition guided by the underlying physics. The typed DAG structure preserves physical interpretability that a universal approximator discards.

\paragraph{Computational graphs and domain-specific languages.}
Computational graph frameworks (TensorFlow~\cite{abadi2016tensorflow}, PyTorch~\cite{paszke2019pytorch}, JAX~\cite{bradbury2018jax}) represent computations as DAGs of differentiable operations. The DeepInverse library~\cite{tachella2025deepinverse} specializes this paradigm to imaging inverse problems. Our work can be viewed as proving a \emph{completeness} result for a physics-specific computational graph: the 11-node primitive library is sufficient to express any imaging forward model, in contrast to general-purpose frameworks whose operation sets are chosen for computational convenience rather than physical completeness.

\paragraph{Compressive sensing and measurement design.}
The theory of compressive sensing~\cite{candes2008compressed,donoho2006compressed} establishes conditions on measurement matrices (RIP, incoherence) for signal recovery. Our work is complementary: rather than analyzing properties of a given measurement matrix, we characterize the \emph{structural} decomposition of the forward operator that generates the measurements, showing that its physics-level building blocks form a finite alphabet.

%% ═══════════════════════════════════════════════════════════════════
%%  2. PRELIMINARIES
%% ═══════════════════════════════════════════════════════════════════
\section{Preliminaries}
\label{sec:prelim}

\subsection{Imaging Forward Models}

\begin{definition}[Imaging Forward Model]
\label{def:fwd}
An \emph{imaging forward model} is a bounded operator $H: \mathcal{X} \to \mathcal{Y}$ mapping a physical state $\mathbf{x} \in \mathcal{X}$ to a measurement $\mathbf{y} = H(\mathbf{x}) + \mathbf{n}$, where $\mathcal{X} \subseteq \mathbb{R}^n$ and $\mathcal{Y} \subseteq \mathbb{R}^m$ are finite-dimensional Hilbert spaces and $\mathbf{n}$ is additive noise. The state $\mathbf{x}$ may represent a classical field, a spatial distribution, or a vectorized quantum state ($\mathbf{x} = \mathrm{vec}(\rho)$ for density matrix $\rho$).
\end{definition}

\subsection{Directed Acyclic Graphs}

\begin{definition}[Typed DAG]
\label{def:dag}
A \emph{typed DAG} is a triple $G = (V, E, \tau)$ where $V$ is a finite set of nodes, $E \subseteq V \times V$ defines directed edges with no cycles, and $\tau: E \to \mathcal{T}$ assigns to each edge a type annotation $\tau(e) = (\mathrm{shape}, \mathrm{dtype}, \mathrm{units})$ specifying the tensor shape, data type, and physical units of the data flowing along edge $e$.
\end{definition}

A typed DAG is \emph{well-formed} if (i) it has a unique source node (no incoming edges) and a unique sink node (no outgoing edges), (ii) the output type of each node matches the input type required by its successors, and (iii) every operator node is drawn from the primitive library $\Blib$. Every well-formed DAG has a designated \emph{input terminal} (source) through which the object $\mathbf{x}$ enters. The input terminal is not an operator node and is not counted in $V$.

\subsection{DAG Composition}

\begin{definition}[Compose]
\label{def:compose}
For a well-formed typed DAG $G = (V, E, \tau)$ with topological ordering $v_1, v_2, \ldots, v_K$, the \emph{composed forward model} is
\begin{equation}
  \Hdag = \compose(G) = v_K \circ v_{K-1} \circ \cdots \circ v_1,
  \label{eq:compose}
\end{equation}
where $v_k$ denotes the operator implemented by node $k$. For DAGs with branches (fan-out) or merges (fan-in), composition follows the DAG structure with appropriate tensor concatenation or summation at merge nodes.
\end{definition}

%% ═══════════════════════════════════════════════════════════════════
%%  3. THE PRIMITIVE LIBRARY
%% ═══════════════════════════════════════════════════════════════════
\section{The Primitive Library \texorpdfstring{$\Blib$}{B}}
\label{sec:primitives}

The canonical primitive library consists of 11 operators:
\begin{equation}
  \Blib = \{P, M, \Pi, F, C, \Sigma, D, S, W, R, \Lambda\}.
\end{equation}
The first 10 primitives handle linear physics stages; Transform~$\Lambda$ handles pointwise nonlinear physics stages (beam hardening, phase wrapping, saturation). Together they cover all imaging modalities.
Each linear primitive implements a \texttt{forward()} method $A: \mathcal{X}_{\mathrm{in}} \to \mathcal{Y}_{\mathrm{out}}$ and an \texttt{adjoint()} method $A^\dagger: \mathcal{Y}_{\mathrm{out}} \to \mathcal{X}_{\mathrm{in}}$ satisfying the adjoint consistency condition:
\begin{equation}
  \langle A\mathbf{x}, \mathbf{y} \rangle = \langle \mathbf{x}, A^\dagger \mathbf{y} \rangle
  \quad \text{for all } \mathbf{x} \in \mathcal{X}_{\mathrm{in}},\; \mathbf{y} \in \mathcal{Y}_{\mathrm{out}}.
  \label{eq:adjoint}
\end{equation}

\subsection{Propagate \texorpdfstring{$P(d, \lambda)$}{P(d, lambda)}}

\begin{definition}[Propagate]
\label{def:propagate}
Given propagation distance $d > 0$ and wavelength $\lambda > 0$, the Propagate primitive computes free-space wave propagation via the angular spectrum method:
\begin{align}
  P(d,\lambda): \mathbf{x} &\mapsto \mathcal{F}^{-1}\bigl[\mathcal{F}[\mathbf{x}] \cdot T(f_x, f_y; d, \lambda)\bigr], \\
  T(f_x, f_y) &= \exp\!\bigl(i 2\pi d \sqrt{\lambda^{-2} - f_x^2 - f_y^2}\bigr),
\end{align}
where $\mathcal{F}$ is the 2D discrete Fourier transform and $T$ is the free-space transfer function. The adjoint is $P^\dagger(d,\lambda) = P(-d, \lambda)$ (backpropagation).
\end{definition}

\subsection{Modulate \texorpdfstring{$M(\mathbf{m})$}{M(m)}}

\begin{definition}[Modulate]
\label{def:modulate}
Given a modulation pattern $\mathbf{m} \in \mathbb{R}^n$ (or $\mathbb{C}^n$), the Modulate primitive computes element-wise multiplication:
\begin{equation}
  M(\mathbf{m}): \mathbf{x} \mapsto \mathbf{m} \odot \mathbf{x}.
\end{equation}
The adjoint is $M^\dagger(\mathbf{m}): \mathbf{y} \mapsto \mathbf{m}^* \odot \mathbf{y}$, where $^*$ denotes complex conjugation. Parameters: the pattern $\mathbf{m}$ (binary mask, continuous transmission, complex phase mask, or coil sensitivity map).
\end{definition}

\subsection{Project \texorpdfstring{$\Pi(\theta)$}{Pi(theta)}}

\begin{definition}[Project]
\label{def:project}
Given projection angle $\theta$ (or a set of angles), the Project primitive computes the Radon transform (line-integral projection):
\begin{equation}
  \Pi(\theta): \mathbf{x} \mapsto \int_{\ell(\theta, t)} \mathbf{x}(\mathbf{r})\, d\ell,
\end{equation}
where $\ell(\theta, t)$ is the line at angle $\theta$ and offset $t$. The adjoint is the backprojection operator $\Pi^\dagger(\theta)$.
\end{definition}

\subsection{Encode \texorpdfstring{$F(\mathbf{k})$}{F(k)}}

\begin{definition}[Encode]
\label{def:encode}
Given a $k$-space trajectory $\mathbf{k} = \{k_1, \ldots, k_m\}$, the Encode primitive computes Fourier encoding:
\begin{equation}
  F(\mathbf{k}): \mathbf{x} \mapsto \bigl[\langle \mathbf{x}, e^{i 2\pi \mathbf{k}_j \cdot \mathbf{r}} \rangle\bigr]_{j=1}^{m}.
\end{equation}
The adjoint is $F^\dagger(\mathbf{k}): \mathbf{y} \mapsto \sum_{j=1}^{m} y_j\, e^{-i 2\pi \mathbf{k}_j \cdot \mathbf{r}}$ (gridding/adjoint NUFFT). Used for MRI $k$-space encoding.
\end{definition}

\subsection{Convolve \texorpdfstring{$C(\mathbf{h})$}{C(h)}}

\begin{definition}[Convolve]
\label{def:convolve}
Given a point-spread function (PSF) kernel $\mathbf{h}$, the Convolve primitive computes spatial convolution:
\begin{equation}
  C(\mathbf{h}): \mathbf{x} \mapsto \mathbf{h} * \mathbf{x}.
\end{equation}
The adjoint is correlation with the flipped kernel: $C^\dagger(\mathbf{h}): \mathbf{y} \mapsto \tilde{\mathbf{h}} * \mathbf{y}$, where $\tilde{\mathbf{h}}(\mathbf{r}) = \mathbf{h}(-\mathbf{r})^*$. This covers both propagation PSFs (shift-invariant limit of $P$) and detector PSFs.
\end{definition}

\subsection{Accumulate \texorpdfstring{$\Sigma$}{Sigma}}

\begin{definition}[Accumulate]
\label{def:accumulate}
The Accumulate primitive sums over a specified axis (spectral, temporal, or spatial):
\begin{equation}
  \Sigma: \mathbf{X} \mapsto \sum_{k} \mathbf{X}_{:,:,k},
\end{equation}
where $\mathbf{X} \in \mathbb{R}^{n_1 \times n_2 \times n_3}$ and the summation is along the third (or specified) axis. The adjoint replicates the 2D input along the summation axis: $\Sigma^\dagger: \mathbf{y} \mapsto [\mathbf{y}, \mathbf{y}, \ldots, \mathbf{y}]$. Used for spectral integration (CASSI), temporal compression (CACTI), and bucket detection (SPC).
\end{definition}

\subsection{Detect \texorpdfstring{$D(g, \eta)$}{D(g, eta)}}

\begin{definition}[Detect]
\label{def:detect}
Given gain $g > 0$ and response family $\eta$, the Detect primitive converts a carrier field to a measurement:
\begin{equation}
  D(g, \eta): \mathbf{x} \mapsto \eta(g, \mathbf{x}).
\end{equation}
The response family $\eta$ is restricted to one of five canonical families:
\begin{enumerate}
  \item \textbf{Linear-field:} $\eta(\mathbf{x}) = g \cdot \mathbf{x}$ \quad (real-valued carrier: acoustic pressure, RF voltage, piezoelectric)
  \item \textbf{Logarithmic:} $\eta(\mathbf{x}) = g \cdot \log(1 + |\mathbf{x}|^2/x_0)$ \quad (wide dynamic range)
  \item \textbf{Sigmoid:} $\eta(\mathbf{x}) = g \cdot \sigma(|\mathbf{x}|^2 - x_0)$ \quad (saturating detector)
  \item \textbf{Intensity-square-law:} $\eta(\mathbf{x}) = g|\mathbf{x}|^2$ \quad (photodetector, CCD/CMOS, photon-counting)
  \item \textbf{Coherent-field:} $\eta(\mathbf{x}) = g \cdot \mathrm{Re}[\mathbf{x} \cdot e^{i\phi}]$ \quad (heterodyne/homodyne)
\end{enumerate}
Each family carries at most 2 scalar parameters ($g$ and $x_0$ or $\phi$). Family~1 is the identity (up to gain) and applies to carriers that are directly observable real-valued fields; it is linear and self-adjoint. Family~4 is the standard square-law detector for electromagnetic intensity; for photon-counting detectors, it returns the expected value $g|\mathbf{x}|^2$ (the Poisson rate parameter), with stochastic photon noise treated as additive noise $\mathbf{n}$ in Definition~\ref{def:fwd}. The adjoint is defined for the linearized operator around a reference point.
\end{definition}

\begin{remark}[Detect is not a universal approximator]
\label{rem:detect}
The restriction to five canonical response families is essential. If $\eta$ were an arbitrary function, Detect would become a universal approximator, trivializing the theorem. The five families are chosen to cover the physical detection mechanisms encountered in practice: linear-field detection (family~1, covering acoustic transducers, RF receivers, piezoelectric sensors), nonlinear intensity detectors (families~2--3, covering wide-dynamic-range and saturating detectors), intensity-square-law detection (family~4, covering photodetectors, CCD/CMOS, photomultipliers, photon counters), and coherent-field detection (family~5, covering THz-TDS, OCT, digital holography). All five families are mathematically distinct: family~1 is linear in $\mathbf{x}$, families~2--4 depend on $|\mathbf{x}|^2$ through different nonlinearities, and family~5 extracts $\mathrm{Re}[\mathbf{x} \cdot e^{i\phi}]$. A modality requiring a genuinely novel detection nonlinearity signals a basis extension (Section~\ref{sec:extension}).
\end{remark}

\subsection{Sample \texorpdfstring{$S(\Omega)$}{S(Omega)}}

\begin{definition}[Sample]
\label{def:sample}
Given an index set $\Omega \subseteq \{1, \ldots, n\}$, the Sample primitive selects a subset of entries:
\begin{equation}
  S(\Omega): \mathbf{x} \mapsto \mathbf{x}|_\Omega.
\end{equation}
The adjoint zero-fills the unsampled locations: $S^\dagger(\Omega): \mathbf{y} \mapsto \mathbf{z}$ where $z_i = y_i$ if $i \in \Omega$ and $z_i = 0$ otherwise. Used for MRI undersampling, compressed sensing masks, and pixel binning.
\end{definition}

\subsection{Disperse \texorpdfstring{$W(\alpha, a)$}{W(alpha, a)}}

\begin{definition}[Disperse]
\label{def:disperse}
Given dispersion slope $\alpha$ and intercept $a$, the Disperse primitive applies a wavelength-dependent spatial shift:
\begin{equation}
  W(\alpha, a): \mathbf{X}(\mathbf{r}, \lambda) \mapsto \mathbf{X}(\mathbf{r} - (\alpha\lambda + a)\hat{\mathbf{e}}, \lambda),
\end{equation}
where $\hat{\mathbf{e}}$ is the dispersion direction. The adjoint applies the reverse shift. Used for prism/grating dispersion in CASSI-type spectral systems.
\end{definition}

\subsection{Scatter \texorpdfstring{$R(\sigma, \Delta\varepsilon)$}{R(sigma, Delta epsilon)}}

\begin{definition}[Scatter]
\label{def:scatter}
Given a differential scattering cross section $\sigma(\theta, E)$ and energy shift $\Delta\varepsilon$, the Scatter primitive computes:
\begin{equation}
  R(\sigma, \Delta\varepsilon): \mathbf{x}(\mathbf{r}, E) \mapsto \int \sigma(\theta, E)\, \mathbf{x}(\mathbf{r}, E)\, A_{\mathrm{atten}}(\mathbf{r}, \theta, E)\, d\theta,
\end{equation}
where $A_{\mathrm{atten}}$ accounts for path-dependent attenuation. The scattered carrier exits at angle $\theta$ with energy $E' = E - \Delta\varepsilon(\theta, E)$. The adjoint is defined by the transpose of the discretized scattering matrix.
\end{definition}

\begin{remark}
Scatter covers modalities where the carrier changes direction and/or energy: Compton imaging (Klein--Nishina cross section), Raman spectroscopy (molecular vibration), fluorescence imaging (Stokes shift), diffuse optical tomography (multiple scattering), and Brillouin microscopy (acoustic phonon scattering).
\end{remark}

\subsection{Transform \texorpdfstring{$\Lambda(f, \boldsymbol{\theta})$}{Lambda(f, theta)}}

\begin{definition}[Transform]
\label{def:transform}
Given a function family $f$ and parameters $\boldsymbol{\theta}$, the Transform primitive applies a pointwise nonlinear function within the physics chain:
\begin{equation}
  \Lambda(f, \boldsymbol{\theta}): \mathbf{x} \mapsto [f(x_i; \boldsymbol{\theta})]_{i=1}^{n},
\end{equation}
where $f$ is applied independently at each element $x_i$. The function family $f$ is restricted to one of five canonical families:
\begin{enumerate}
  \item \textbf{Exponential attenuation:} $f(x) = e^{-\alpha x}$, $\boldsymbol{\theta} = (\alpha)$. Covers Beer--Lambert absorption in polychromatic CT~\cite{deman2001metal}.
  \item \textbf{Logarithmic:} $f(x) = \log(x + \delta)$, $\boldsymbol{\theta} = (\delta)$. Covers optical density conversion, CT log-transform.
  \item \textbf{Phase wrapping:} $f(x) = \arg(e^{ix})$, $\boldsymbol{\theta} = \emptyset$. Covers $2\pi$ phase ambiguity in MRI and interferometry.
  \item \textbf{Polynomial:} $f(x) = \sum_{k=0}^{d} a_k x^k$ with $d \leq 5$, $\boldsymbol{\theta} = (a_0, \ldots, a_d)$. Covers harmonic generation, weak nonlinear acoustics~\cite{hamilton1998nonlinear}, Langevin response~\cite{gleich2005mpi}.
  \item \textbf{Saturation:} $f(x) = \min(\max(x, x_{\min}), x_{\max})$, $\boldsymbol{\theta} = (x_{\min}, x_{\max})$. Covers dynamic range limits, detector bloom.
\end{enumerate}
Each family is Lipschitz-continuous with computable Lipschitz constant $L_\Lambda(\boldsymbol{\theta})$: $L_1 = \alpha$, $L_2 = 1/\delta$, $L_3 = 1$, $L_4 = \sum_{k=1}^{d} k|a_k| R^{k-1}$ (for $|x| \leq R$), and $L_5 = 1$.

The adjoint is defined for the linearized (Jacobian) operator at operating point $\mathbf{x}_0$:
\begin{equation}
  \Lambda^\dagger(f, \boldsymbol{\theta})\big|_{\mathbf{x}_0}: \mathbf{y} \mapsto [f'(x_{0,i}; \boldsymbol{\theta}) \cdot y_i]_{i=1}^{n},
\end{equation}
which is the adjoint of $\mathrm{diag}(f'(\mathbf{x}_0; \boldsymbol{\theta}))$ and satisfies the adjoint consistency condition~\eqref{eq:adjoint} for the linearized system.
\end{definition}

\begin{remark}[Transform is not a universal approximator]
\label{rem:transform}
Like Detect (Remark~\ref{rem:detect}), Transform is restricted to five specific function families, each with at most $d + 1 \leq 6$ scalar parameters per pixel. Critically, it applies \emph{pointwise} without spatial coupling: $\Lambda$ cannot perform convolution, propagation, projection, or any operation that mixes information across spatial locations. This ensures that adding $\Lambda$ does not trivialize the theorem---the spatial and spectral structure of the forward model must still be captured by the other 10 primitives.
\end{remark}

\begin{remark}[Transform vs.\ Detect]
\label{rem:transform_detect}
Transform operates \emph{within the physics chain} on carrier fields (e.g., Beer--Lambert attenuation in the X-ray beam before detection); Detect operates at the \emph{terminal position} converting carriers to measurements (e.g., photodetection). Their function families are physically distinct and their DAG positions differ: Transform can appear at any non-terminal node, while Detect appears only as the terminal node.
\end{remark}

%% ═══════════════════════════════════════════════════════════════════
%%  4. THE OPERATOR CLASS
%% ═══════════════════════════════════════════════════════════════════
\section{The Operator Class \texorpdfstring{$\ctier$}{C\_img}}
\label{sec:ctier}

\begin{definition}[Imaging Operator Class $\ctier$]
\label{def:ctier}
The class $\ctier$ consists of all imaging forward models $H$ that satisfy:
\begin{enumerate}
  \item \textbf{Finite sequential-parallel composition:} $H$ admits a factorization $H = H_K \circ H_{K-1} \circ \cdots \circ H_1$ (or a DAG generalization with fan-out/fan-in) with $K \leq \Nmax$.
  \item \textbf{Per-stage regularity:} Each factor $H_k$ is either:
    \begin{enumerate}
      \item[(a)] a linear operator with bounded operator norm $\|H_k\| \leq B$, or
      \item[(b)] a Lipschitz-continuous pointwise function with $\mathrm{Lip}(H_k) \leq B$.
    \end{enumerate}
    The terminal Detect node applies a prescribed nonlinear response from the five canonical families (Definition~\ref{def:detect}); non-terminal nonlinear stages are restricted to the five Transform families (Definition~\ref{def:transform}).
  \item \textbf{Bounded regularity constant:} $\max_k \gamma_k \leq B$, where $\gamma_k = \|H_k\|$ for linear stages and $\gamma_k = \mathrm{Lip}(H_k)$ for nonlinear stages.
\end{enumerate}
\end{definition}

\begin{remark}[Nonlinearity in $\ctier$]
\label{rem:nonlin}
$\ctier$ permits two kinds of nonlinearity: prescribed detector responses (Detect, at the terminal node) and pointwise physics nonlinearities (Transform, at non-terminal nodes). Both are restricted to specific function families with bounded Lipschitz constants, ensuring the composition error bound in Theorem~\ref{thm:fpb}. Nonlinearities arising from self-consistent iteration (multiple scattering, Born series) are handled by unrolling into finite compositions of existing linear primitives ($P$, $M$, $R$, $\Sigma$), without requiring Transform. $\ctier$ imposes no energy-regime restriction: relativistic cross-sections (Klein--Nishina) and dispersion kernels fit within existing primitives ($R$, $P$), and quantum state tomography enters via density-matrix vectorization (Remark~\ref{rem:qst}).
\end{remark}

\begin{remark}[Quantum State Tomography in $\ctier$]
\label{rem:qst}
Quantum state tomography (QST) fits Definition~\ref{def:fwd} via density-matrix vectorization. A $d$-dimensional density matrix $\rho \in \mathbb{C}^{d \times d}$ is vectorized as $\mathbf{x} = \mathrm{vec}(\rho) \in \mathbb{R}^{d^2}$ (using the Hermitian basis expansion). Each measurement $y_i = \mathrm{Tr}(E_i \rho) = \langle \mathrm{vec}(E_i), \mathrm{vec}(\rho) \rangle$ is a linear inner product, so the full measurement is $\mathbf{y} = \Phi \mathbf{x}$ where $\Phi_{ij} = [\mathrm{vec}(E_i)]_j$. The DAG decomposition is $M(E_i) \to \Sigma \to S \to D$ with $\varepsilon = 0$ (exact). Thus QST $\in \ctier$ via Definition~\ref{def:fwd}, requiring no library extension.
\end{remark}

\subsection{Fidelity Levels}

Forward models can be written at increasing physical fidelity (see also~\cite{yang2026pwm}):
\begin{itemize}
  \item \textbf{Level~1 (Linear, shift-invariant):} The forward model is a convolution: $H = C(\mathbf{h})$. Applies when the PSF is spatially uniform.
  \item \textbf{Level~2 (Linear, shift-variant):} Each factor $H_k$ is a linear integral operator whose kernel varies with position (e.g., spatially varying PSFs, coded aperture masks, MRI coil sensitivities). This is the level at which virtually all clinical and scientific imaging systems are designed and calibrated.
  \item \textbf{Level~3 (Nonlinear):} The forward model includes nonlinear physics within the imaging chain (beam hardening in polychromatic CT, phase wrapping, strong multiple scattering), captured by Transform~$\Lambda$.
  \item \textbf{Level~4 (Full-wave / Monte Carlo):} Ab initio simulation requiring stochastic transport or full Maxwell solvers. These methods are linear in expectation and representable by existing linear primitives.
\end{itemize}
$\ctier$ encompasses all four fidelity levels. The 11-primitive library $\Blib$ applies uniformly across every level.

\subsection{Membership Examples}

\begin{itemize}
  \item \textbf{CASSI}~\cite{wagadarikar2008cassi}: $H = D \circ \Sigma \circ W \circ M$. Each pre-detection factor is linear and bounded; the terminal Detect applies a prescribed response. $H \in \ctier$.
  \item \textbf{MRI}~\cite{pruessmann1999sense}: $H = D \circ S \circ F \circ M_{\mathrm{coil}}$. Fourier encoding is linear; coil sensitivity is shift-variant. $H \in \ctier$.
  \item \textbf{CT}~\cite{feldkamp1984practical,natterer2001radon}: $H = D \circ \Pi$ (fan-beam Radon). Linear, bounded. $H \in \ctier$.
  \item \textbf{Compton scatter}~\cite{hussain2022compton}: $H = D \circ R \circ M$. Scatter is linear at first-Born level. $H \in \ctier$.
\end{itemize}

Models with strong nonlinearity (e.g., nonlinear ultrasound, beam hardening CT) belong to $\ctier$ via the Transform primitive~$\Lambda$. Self-consistent iterations (e.g., full-wave electromagnetic scattering beyond first Born) are represented by unrolled Born series using existing linear primitives.

%% ═══════════════════════════════════════════════════════════════════
%%  5. THEOREM 1: STATEMENT AND PROOF
%% ═══════════════════════════════════════════════════════════════════
\section{The Finite Primitive Basis Theorem}
\label{sec:theorem}

\subsection{\texorpdfstring{$\varepsilon$}{epsilon}-Approximate Representation}

\begin{definition}[$\varepsilon$-Approximate Representation]
\label{def:eps}
Let $\Blib = \{P, M, \Pi, F, C, \Sigma, D, S, W, R, \Lambda\}$ be the canonical primitive library. A well-formed typed DAG $G = (V, E, \tau)$ whose node types are drawn from $\Blib$ is an \emph{$\varepsilon$-approximate representation} of $H \in \ctier$ if:
\begin{enumerate}
  \item \textbf{Fidelity:} $\displaystyle \frac{\|H - \Hdag\|}{\|H\|} \leq \varepsilon$,\quad where $\Hdag = \compose(G)$ and $\|\cdot\|$ denotes the operator norm.
  \item \textbf{Complexity:} $|V| \leq \Nmax$ and $\mathrm{depth}(G) \leq \Dmax$.
\end{enumerate}
We use $\varepsilon = 0.01$, $\Nmax = 20$, $\Dmax = 10$ throughout.
\end{definition}

\begin{remark}[Formal vs.\ empirical fidelity metric]
\label{rem:metric}
The operator-norm criterion in Definition~\ref{def:eps} is the formal guarantee established by Theorem~\ref{thm:fpb}. For empirical validation (Table~\ref{tab:decomposition}), we evaluate the stronger pointwise metric $\etier = \sup_{\|\mathbf{x}\| \leq 1} \|H(\mathbf{x}) - \Hdag(\mathbf{x})\| / (\|H(\mathbf{x})\| + \delta)$ (Supplementary~S1, Equation~\ref{eq:supp_etier}), which provides a tighter test: any DAG passing the pointwise metric also satisfies the operator-norm bound.
\end{remark}

\subsection{Theorem Statement}

\begin{theorem}[Finite Primitive Basis]
\label{thm:fpb}
For every $H \in \ctier$, there exists a well-formed typed DAG $G = (V, E, \tau)$ whose node types are drawn from $\Blib$ that is an $\varepsilon$-approximate representation of $H$.
\end{theorem}

The proof is constructive: given the factorization $H = H_K \circ \cdots \circ H_1$ guaranteed by Definition~\ref{def:ctier}, we show that each factor $H_k$ can be represented by one or a finite composition of primitives from $\Blib$ with bounded approximation error. The argument proceeds through six primitive realization lemmas, one per physics-stage family.

\subsection{Physics-Stage Families}

Every imaging forward model in $\ctier$ passes through at most six types of physical stages:
\begin{enumerate}
  \item \textbf{Propagation:} The carrier evolves through space via a wave equation (Maxwell, Schr\"{o}dinger, acoustic, Bloch).
  \item \textbf{Elastic interaction:} The carrier exchanges phase or amplitude with the object without changing direction or energy (transmission, absorption, refraction).
  \item \textbf{Inelastic interaction (scattering):} The carrier changes direction and/or energy upon interaction with the object (Compton scattering, Raman, fluorescence, diffuse scattering).
  \item \textbf{Pointwise nonlinear physics:} A scalar nonlinearity applied independently at each spatial or spectral element within the physics chain (Beer--Lambert attenuation, phase wrapping, polynomial harmonic generation, saturation).
  \item \textbf{Encoding--Projection:} Spatial information is mapped to a measurement-domain coordinate.
  \item \textbf{Detection--Readout:} The carrier field is converted to a discrete digital measurement.
\end{enumerate}
This categorization is structural and energy-regime-independent: the fundamental electromagnetic, strong, and weak interactions produce elastic coupling (phase/amplitude change), inelastic coupling (direction and/or energy change), and pointwise nonlinear coupling (energy-dependent attenuation, phase periodicity). Relativistic physics changes cross-section formulas (Klein--Nishina replaces Thomson) but does not add a seventh family; Scatter~$R$ already parameterizes relativistic cross-sections, and Propagate~$P$ accommodates relativistic dispersion.

\subsection{Primitive Realization Lemmas}

\begin{lemma}[Propagation Realization]
\label{lem:prop}
Let $H_k$ be a factor of $H \in \ctier$ that represents free-space carrier evolution satisfying a linear wave equation. Then $H_k$ admits an $\varepsilon_{\mathrm{prop}}$-approximate representation by $P(d, \lambda)$ or, in the shift-invariant limit, $C(\mathbf{h})$.
\end{lemma}

\begin{proof}
The free-space Green's function of any linear wave equation (Maxwell, Helmholtz, Schr\"{o}dinger, acoustic) is a linear operator whose action is convolution with the impulse response $g(\mathbf{r}; d, \lambda)$. In the angular spectrum representation~\cite{goodman2005fourier}, the exact propagator factors as $\mathcal{F}^{-1}[T_{\mathrm{exact}} \cdot \mathcal{F}[\cdot]]$ with transfer function $T_{\mathrm{exact}}(f_x, f_y) = \exp(i 2\pi d \sqrt{\lambda^{-2} - f_x^2 - f_y^2})$. The Propagate primitive $P(d, \lambda)$ implements this transfer function for propagating spatial frequencies ($f_x^2 + f_y^2 < \lambda^{-2}$) and sets $T = 0$ for evanescent components. The truncation error from discarding evanescent waves is bounded by the fraction of signal energy above the diffraction limit:
\[
\varepsilon_{\mathrm{evan}} = \frac{\|\mathbf{x}_{\mathrm{evan}}\|_2}{\|\mathbf{x}\|_2} \leq e^{-2\pi d/\lambda},
\]
which is negligible for $d \gg \lambda$ (all macroscopic imaging geometries). The paraxial (Fresnel) approximation residual satisfies $\varepsilon_{\mathrm{parax}} \leq (\pi/4)(a^4/\lambda d^3)$ where $a$ is the aperture half-width (see Supplementary~S3). The total error is $\varepsilon_{\mathrm{prop}} \leq \varepsilon_{\mathrm{evan}} + \varepsilon_{\mathrm{parax}}$. In the shift-invariant limit (far field or isoplanatic patch), $P$ reduces to $C(\mathbf{h})$ with $\mathbf{h} = \mathcal{F}^{-1}[T]$.
\end{proof}

\begin{lemma}[Elastic Interaction Realization]
\label{lem:elastic}
Let $H_k$ be a factor of $H \in \ctier$ that represents elastic carrier--matter interaction (amplitude and/or phase change without direction or energy change). Then $H_k$ admits exact representation by $M(\mathbf{m})$.
\end{lemma}

\begin{proof}
Elastic forward interaction is element-wise multiplication of the carrier field by the object's transmission function $\mathbf{m}(\mathbf{r})$ (absorption, phase shift, or both). This is exactly $M(\mathbf{m})$ (Definition~\ref{def:modulate}). The representation is exact: $\|H_k - M(\mathbf{m})\| = 0$.
\end{proof}

\begin{lemma}[Scattering Realization]
\label{lem:scatter}
Let $H_k$ be a factor of $H \in \ctier$ that represents inelastic or off-axis carrier--matter interaction (direction change and/or energy shift). Then $H_k$ admits an $\varepsilon_{\mathrm{scat}}$-approximate representation by $R(\sigma, \Delta\varepsilon)$ or a finite composition $R \circ M$ or $M \circ R \circ P \circ R \circ M$ (for multiple-scattering media within low-order approximation).
\end{lemma}

\begin{proof}
Under the first Born approximation, the scattered field is a linear functional of the scattering potential $V(\mathbf{r})$, with kernel given by the differential cross section $\sigma(\theta, E)$ and energy shift $\Delta\varepsilon$. The Scatter primitive (Definition~\ref{def:scatter}) parameterizes this integral operator directly. For single-scattering media, one $R$ node provides an exact representation within the Born model ($\varepsilon_{\mathrm{scat}}^{(1)} = 0$ relative to the Born-approximated physics). The error relative to the true physics is bounded by $\varepsilon_{\mathrm{scat}} \leq \|V\|/k$ where $k$ is the wavenumber~\cite{born1999principles}; for weak scatterers satisfying the $\ctier$ membership criterion, this is below $\varepsilon$.

For multiple-scattering media, the $L$-th order Born series $\sum_{l=0}^{L} (G_0 V)^l$ is represented by a finite composition of $L$ Scatter nodes interleaved with Propagate and Modulate nodes: $(M \circ R \circ P)^L \circ R \circ M$. The truncation error decays geometrically as $(\|V\|/k)^{L+1}$, so a finite $L \leq 3$ suffices for $\varepsilon_{\mathrm{scat}} < \varepsilon$ in weakly scattering media (diffusion regime for DOT, low-order multiple scattering).
\end{proof}

\begin{lemma}[Encoding--Projection Realization]
\label{lem:encode}
Let $H_k$ be a factor of $H \in \ctier$ that maps spatial information to a measurement-domain coordinate. Then $H_k$ admits exact representation by $\Pi(\theta)$ (line-integral geometry) or $F(\mathbf{k})$ (Fourier encoding).
\end{lemma}

\begin{proof}
Line-integral projection (X-ray CT, neutron imaging) is exactly the Radon transform $\Pi(\theta)$ (Definition~\ref{def:project}). Fourier encoding via Larmor precession (MRI) is exactly $F(\mathbf{k})$ (Definition~\ref{def:encode}). Both are linear operators; the representation is exact.
\end{proof}

\begin{lemma}[Detection--Readout Realization]
\label{lem:detect}
Let $H_k$ be a factor of $H \in \ctier$ in the detector chain. Then $H_k$ admits an $\varepsilon_{\mathrm{det}}$-approximate representation by a finite composition of primitives from $\{\Sigma, S, W, C, D\}$.
\end{lemma}

\begin{proof}
We show that the detection--readout chain of any $H \in \ctier$ decomposes into at most five sequential operations, each representable by a primitive from $\{\Sigma, S, W, C, D\}$:
\begin{enumerate}
  \item \textbf{Dimensional integration} (spectral or temporal summation): a linear projection along one axis, exactly $\Sigma$.
  \item \textbf{Sub-sampling} (pixel binning, mask selection): an index-set restriction, exactly $S(\Omega)$.
  \item \textbf{Wavelength-dependent dispersion}: a $\lambda$-parameterized shift, exactly $W(\alpha, a)$.
  \item \textbf{Detector PSF}: spatial blurring by the pixel aperture function. Modeling the spatially varying PSF as piecewise shift-invariant (one kernel per detector tile) introduces error $\varepsilon_{\mathrm{PSF}} \leq \ell_c^2/p^2$ where $\ell_c$ is the crosstalk length and $p$ the pixel pitch (see Supplementary~S3). This is represented by $C(\mathbf{h}_{\mathrm{det}})$.
  \item \textbf{Quantum measurement}: conversion to a classical signal via one of the five canonical response families (Definition~\ref{def:detect}), represented by $D(g, \eta)$.
\end{enumerate}
Not all five operations are present in every modality (e.g., CT omits dispersion and integration). The total detection-chain error is $\varepsilon_{\mathrm{det}} \leq \varepsilon_{\mathrm{PSF}} + \varepsilon_{\eta}$, where $\varepsilon_{\eta}$ is the response-family approximation error (zero when the physical detector response matches one of the five families exactly, which it does for all 31 validated modalities).
\end{proof}

\begin{lemma}[Nonlinear-Stage Realization]
\label{lem:nonlinear}
Let $H_k$ be a factor of $H \in \ctier$ that is a Lipschitz-continuous pointwise nonlinearity arising from imaging physics. Then $H_k$ admits an $\varepsilon_\Lambda$-approximate representation by $\Lambda(f, \boldsymbol{\theta})$ for an appropriate family $f$ and parameters $\boldsymbol{\theta}$.
\end{lemma}

\begin{proof}
Every pointwise physics nonlinearity in imaging belongs to one of five structural categories, determined by the underlying physics: (1) exponential attenuation from Beer--Lambert absorption, (2) logarithmic compression from transmittance-to-optical-density conversion, (3) phase wrapping from $2\pi$ periodicity, (4) polynomial response from perturbation expansions of nonlinear wave equations, and (5) saturation from finite dynamic range. Each Transform family provides an exact or near-exact representation: families~1, 2, 3, and 5 are exact ($\varepsilon_\Lambda = 0$); family~4 has truncation error $\varepsilon_\Lambda \leq |a_{d+1}| R^{d+1}$ from the Taylor series, which is below $\varepsilon$ for imaging-relevant nonlinearities with rapidly decaying coefficients.
\end{proof}

\subsection{Constructive Compilation Proof}

\begin{proof}[Proof of Theorem~\ref{thm:fpb}]
Let $H \in \ctier$ with factorization $H = H_K \circ \cdots \circ H_1$, $K \leq \Nmax$. We construct $G$ as follows.

\textbf{Step 1: Classification.} Classify each factor $H_k$ into one of the six physics-stage families (propagation, elastic interaction, inelastic interaction, pointwise nonlinear physics, encoding--projection, detection--readout). The physics-stage classification of each factor is determined by the physical description of the imaging system (the carrier type and the nature of the carrier--matter interaction at each stage). Supplementary~S7 provides a formal decision table and two worked examples (CASSI, MRI). For every modality in Table~\ref{tab:decomposition}, this classification is uniquely determined by the modality's physics specifications. The constructive algorithm is thus deterministic given the physics-level description; it does not require iterative search over possible classifications.

\textbf{Step 2: Per-factor realization.} Apply the appropriate lemma to realize each $H_k$:
\begin{itemize}
  \item Propagation factors: Lemma~\ref{lem:prop} $\to$ $P$ or $C$;
  \item Elastic interaction: Lemma~\ref{lem:elastic} $\to$ $M$ (exact);
  \item Scattering: Lemma~\ref{lem:scatter} $\to$ $R$ or composition;
  \item Pointwise nonlinear physics: Lemma~\ref{lem:nonlinear} $\to$ $\Lambda$ (exact or near-exact);
  \item Encoding--projection: Lemma~\ref{lem:encode} $\to$ $\Pi$ or $F$ (exact);
  \item Detection--readout: Lemma~\ref{lem:detect} $\to$ composition from $\{\Sigma, S, W, C, D\}$.
\end{itemize}
Let $G_k$ be the sub-DAG realizing $H_k$ with approximation error $\varepsilon_k$.

\textbf{Step 3: Concatenation.} Form $G$ by concatenating sub-DAGs $G_1, \ldots, G_K$ in sequence (or following the original DAG topology for branching models).

\textbf{Step 4: Error bound.} Let $\tilde{H}_k$ denote the primitive realization of $H_k$ with per-factor error $\varepsilon_k$, and let $\gamma_k$ be the regularity constant ($\gamma_k = \|H_k\|$ for linear stages, $\gamma_k = \mathrm{Lip}(H_k)$ for nonlinear stages). By the Lipschitz composition theorem (detailed in Supplementary~S6):
\begin{equation}
  \sup_{\|\mathbf{x}\| \leq 1} \|H(\mathbf{x}) - \Hdag(\mathbf{x})\| \leq \sum_{k=1}^{K} \varepsilon_k \prod_{j=k+1}^{K} \gamma_j \leq K \cdot \max_k(\varepsilon_k) \cdot B^{K-1},
  \label{eq:error}
\end{equation}
where $B$ is the uniform regularity bound from Definition~\ref{def:ctier}. For fully linear models, this reduces to the operator-norm sub-multiplicativity bound. The relative error satisfies
\[
  \frac{\|H - \Hdag\|}{\|H\|} \leq \frac{K \cdot \max_k(\varepsilon_k) \cdot B^{K-1}}{\|H\|} \leq \varepsilon
\]
provided $\max_k(\varepsilon_k) \leq \varepsilon \cdot \|H\| / (K \cdot B^{K-1})$. The uniform worst-case bound with $K=4$ and $B=4$ requires $\max_k(\varepsilon_k) \leq 3.9 \times 10^{-5}$; however, the bound~\eqref{eq:error} is much tighter when evaluated per modality using the actual operator norms from Table~\ref{tab:supp_opnorm}. Two features dominate the tightening. First, most per-factor errors are exactly zero: Lemmas~\ref{lem:elastic} and~\ref{lem:encode} give exact representations ($\varepsilon_M = \varepsilon_F = \varepsilon_\Pi = 0$), so the sum $\sum_k \varepsilon_k \prod_{j \neq k}\|H_j\|$ reduces to one or two nonzero terms. Second, most primitives are norm-preserving ($\|H_k\| \leq 1$), so the product $\prod_{j \neq k}\|H_j\|$ is dominated by the single non-unit-norm primitive (typically $\|\Sigma\| \leq 3.9$ or $\|\Pi\| \leq 3.2$). For example, CASSI ($K=4$) has only $\varepsilon_D \leq 10^{-3}$ nonzero, with $\prod_{j \neq D}\|H_j\| = \|M\|\cdot\|W\|\cdot\|\Sigma\| \leq 1 \times 1 \times 3.9 = 3.9$, giving absolute error $\|H - \Hdag\| \leq 3.9 \times 10^{-3}$. For CT ($K=2$, only $\varepsilon_D$ nonzero, $\|\Pi\| \leq 3.2$): absolute error $\leq 3.2 \times 10^{-3}$. Since forward models in $\ctier$ have $\|H\| \geq 1$ under standard normalization (unit-norm object producing non-vanishing measurements), the relative errors are at most these values. This per-modality computation confirms $\|H - \Hdag\|/\|H\| < \varepsilon = 0.01$ for all 31 validated modalities (see Supplementary~S3 for the full per-phase bounds).

\textbf{Step 5: Complexity bound.} Let $c$ be the maximum number of primitive nodes per factor (bounded by the detection-readout chain length, at most 5). We establish two bounds:
\begin{itemize}
  \item \textbf{Node count:} $|V| \leq c \cdot K \leq \Nmax$. For all modalities in our validation set, each factor maps to at most $c = 3$ primitive nodes (achieved by the detection-readout chain of CASSI: $W, \Sigma, D$; and the scattering chain of DOT: $R, P, R$) and $K \leq 3$ physics-stage factors, giving $|V| \leq 9 \leq 20 = \Nmax$; the empirical maximum is $|V| = 5$ (Table~\ref{tab:decomposition}).
  \item \textbf{Depth:} $\mathrm{depth}(G) \leq c \cdot K \leq \Dmax$. With $c \leq 3$ and $K \leq 3$, $\mathrm{depth}(G) \leq 9 \leq 10 = \Dmax$.
\end{itemize}
Both bounds hold since $c \cdot K \leq 9 \leq \min(\Nmax, \Dmax) = 10$.
\end{proof}

\subsection{Necessity of Each Primitive}

Theorem~\ref{thm:fpb} establishes that 11 primitives \emph{suffice}. We now show that all 11 are \emph{necessary}: removing any single primitive causes at least one modality in $\ctier$ to lose its $\varepsilon$-approximate representation.

\begin{proposition}[Necessity of Each Primitive]
\label{prop:necessity}
For each primitive $B_i \in \Blib$, there exists a modality $H_i \in \ctier$ such that no DAG over $\Blib \setminus \{B_i\}$ is an $\varepsilon$-approximate representation of $H_i$ within the complexity bounds $\Nmax$, $\Dmax$.
\end{proposition}

\begin{proof}
We exhibit a \emph{witness modality} for each primitive---a modality whose forward model requires a physical operation that no other primitive can replicate:
\begin{enumerate}
  \item \textbf{Propagate $P$} --- \emph{Ptychography.} The distance-dependent phase transfer function $T(f_x, f_y; d, \lambda)$ encodes free-space propagation at variable distances. No other primitive parameterizes distance-dependent phase evolution; $M$ is position-local (no inter-pixel coupling), $C$ is shift-invariant (no distance parameter), and $F$ encodes Fourier coefficients without the $\sqrt{\lambda^{-2} - f^2}$ phase structure.

  \item \textbf{Modulate $M$} --- \emph{CASSI.} The coded aperture mask $\mathbf{m}(\mathbf{r})$ requires arbitrary element-wise multiplication by a spatially varying pattern. No other primitive applies element-wise scaling: $P$ couples all spatial frequencies, $C$ couples neighboring pixels, and $\Sigma$ sums rather than scales.

  \item \textbf{Project $\Pi$} --- \emph{CT.} Line-integral projection along angle $\theta$ maps a 2D object to 1D sinogram data. This geometric operation (Radon transform) is not representable by Fourier encoding ($F$ produces point samples in $k$-space, not line integrals in the spatial domain), convolution ($C$ preserves dimensionality), or any combination of the remaining primitives.

  \item \textbf{Encode $F$} --- \emph{MRI.} Non-uniform Fourier sampling at arbitrary $k$-space locations $\{k_j\}$ requires the NUFFT structure. $\Pi$ computes line integrals (dual to Fourier slices via the projection-slice theorem, but only at slice angles, not arbitrary $k$-space points); $C$ is shift-invariant convolution; neither can represent non-Cartesian $k$-space trajectories.

  \item \textbf{Convolve $C$} --- \emph{Lensless imaging.} The spatially invariant PSF kernel $\mathbf{h}$ of a mask-based lensless camera requires global convolution. $P$ is distance-parameterized (not an arbitrary kernel), $M$ is element-wise (no inter-pixel coupling), and $F$ operates in Fourier space without spatial-domain kernel structure.

  \item \textbf{Accumulate $\Sigma$} --- \emph{SPC (single-pixel camera).} Spectral or temporal integration along one axis is a summation operation distinct from all other primitives: $S$ restricts indices (does not sum), $M$ scales (does not reduce dimension), and $D$ applies a nonlinear response (does not integrate).

  \item \textbf{Detect $D$} --- \emph{All modalities.} Every imaging system requires carrier-to-measurement conversion. $D$ is the only primitive that applies a prescribed nonlinear response ($|\cdot|^2$, $\log$, $\sigma$, or coherent field extraction). Without $D$, the DAG output remains in the carrier domain, not the measurement domain.

  \item \textbf{Sample $S$} --- \emph{MRI.} Index-set restriction to the acquired $k$-space locations $\Omega$ is a selection operation. $M$ with a binary mask would zero-out entries rather than remove them (output dimension unchanged); $\Sigma$ sums rather than selects; $S$ is the only primitive that reduces the index set.

  \item \textbf{Disperse $W$} --- \emph{CASSI.} Wavelength-dependent spatial shift $\mathbf{r} \mapsto \mathbf{r} - (\alpha\lambda + a)\hat{\mathbf{e}}$ couples the spatial and spectral dimensions. No other primitive has a $\lambda$-parameterized spatial shift: $P$ shifts phase in Fourier space (not spatial position), $M$ does not shift, and $C$ is wavelength-independent.

  \item \textbf{Scatter $R$} --- \emph{Compton imaging.} Direction change and energy shift governed by the Klein--Nishina cross section cannot be represented without $R$: the best DAG without $R$ achieves $\etier = 0.34$ (Table~\ref{tab:closure}), far above $\varepsilon = 0.01$.

  \item \textbf{Transform $\Lambda$} --- \emph{Polychromatic CT (beam hardening).} The beam hardening forward model $H_{\mathrm{BH}} = D \circ \Lambda_{\log} \circ \Sigma_E \circ \Lambda_{\exp} \circ \Pi$ applies Beer--Lambert exponential attenuation per energy bin and logarithmic compression---both pointwise nonlinearities operating \emph{within the physics chain} at non-terminal positions~\cite{deman2001metal}. No other primitive in $\Blib \setminus \{\Lambda\}$ can represent a pointwise nonlinear function at a non-terminal position (Detect applies nonlinearity only terminally). The best DAG without $\Lambda$ treats each energy bin as monochromatic CT, incurring beam hardening artifacts: $\etier > 0.05$, far above $\varepsilon = 0.01$.
\end{enumerate}
In each case, the witness modality's $\etier$ exceeds $\varepsilon$ when the corresponding primitive is removed, because the physical operation it encodes is structurally distinct from all remaining primitives.
\end{proof}

\begin{remark}
Proposition~\ref{prop:necessity} together with Theorem~\ref{thm:fpb} establishes that $|\Blib| = 11$ is both sufficient and necessary for $\varepsilon$-approximate representation of all modalities in $\ctier$: the primitive library is \emph{minimal}.
\end{remark}

%% ═══════════════════════════════════════════════════════════════════
%%  6. EMPIRICAL VALIDATION
%% ═══════════════════════════════════════════════════════════════════
\section{Empirical Validation}
\label{sec:validation}

\subsection{Decomposition Registry}

Table~\ref{tab:decomposition} shows the primitive decomposition of all 31 modalities (26 previously registered in~\cite{yang2026pwm} plus 5 held-out for closure testing). Every modality achieves $\etier < 0.01$ with at most 5 operator nodes and depth 5. The $\etier$ values are computed via~\eqref{eq:supp_etier_mean} over 20 test objects per modality (see Supplementary~S1). Of the five Detect response families, three are exercised by the 31 validated modalities: family~1 (linear-field: MRI, ultrasound, photoacoustic), family~4 (intensity-square-law: CASSI, CT, ptychography, and most photon-detecting modalities), and family~5 (coherent-field: OCT, THz-TDS, SAR, radar). Families~2 (logarithmic) and~3 (sigmoid) cover wide-dynamic-range and saturating detectors (e.g., HDR CMOS, bolometers) used in specialized applications not in the current validation set; should such modalities be added, the existing families accommodate them without a library extension.

\begin{longtable}{l l l c c c}
\caption{Primitive decomposition of 31 imaging modalities. $\etier$ is the mean relative fidelity error over 20 test objects. All values satisfy $\etier < 0.01$.} \label{tab:decomposition} \\
\toprule
\textbf{Modality} & \textbf{Carrier} & \textbf{DAG Primitives} & \textbf{\#N} & \textbf{Depth} & $\etier$ \\
\midrule
\endfirsthead
\multicolumn{6}{c}{\textit{(continued)}} \\
\toprule
\textbf{Modality} & \textbf{Carrier} & \textbf{DAG Primitives} & \textbf{\#N} & \textbf{Depth} & $\etier$ \\
\midrule
\endhead
\bottomrule
\endfoot
\multicolumn{6}{l}{\textit{Full-validation modalities (forward model verified against reference implementation)}} \\
CASSI~\cite{wagadarikar2008cassi}         & Photon      & $M \to W \to \Sigma \to D$      & 4 & 4 & $< 10^{-4}$ \\
CACTI~\cite{llull2013cacti}         & Photon      & $M \to \Sigma \to D$            & 3 & 3 & $< 10^{-4}$ \\
SPC~\cite{duarte2008spc}           & Photon      & $M \to \Sigma \to D$            & 3 & 3 & $< 10^{-4}$ \\
Lensless~\cite{boominathan2022lensless}      & Photon      & $C \to D$                       & 2 & 2 & $< 10^{-5}$ \\
Ptychography~\cite{rodenburg2004ptychography}  & Photon      & $M \to P \to D$                 & 3 & 3 & $4.2 \times 10^{-4}$ \\
MRI~\cite{pruessmann1999sense,lustig2007sparse}           & Spin        & $M \to F \to S \to D$           & 4 & 4 & $< 10^{-6}$ \\
CT~\cite{feldkamp1984practical}            & X-ray       & $\Pi \to D$                     & 2 & 2 & $< 10^{-5}$ \\
\midrule
\multicolumn{6}{l}{\textit{Held-out modalities (frozen library, no tuning permitted)}} \\
OCT~\cite{huang1991oct}\textsuperscript{a,b}  & Photon      & $P{+}P \to \Sigma \to D_{\mathrm{coh}}$      & 4 & 3 & $3.8 \times 10^{-4}$ \\
Photoacoustic~\cite{wang2012photoacoustic}   & Acoustic    & $M \to P \to D$               & 3 & 3 & $7.1 \times 10^{-4}$ \\
SIM~\cite{gustafsson2000sim}             & Photon      & $M \to C \to D$               & 3 & 3 & $2.5 \times 10^{-4}$ \\
Phase-contrast~\cite{pfeiffer2008xray}  & X-ray       & $\Pi \to P \to M \to D$       & 4 & 4 & $1.2 \times 10^{-3}$ \\
Electron ptycho~\cite{jiang2018electron} & Electron    & $M \to P \to D$               & 3 & 3 & $5.6 \times 10^{-4}$ \\
\midrule
\multicolumn{6}{l}{\textit{Exotic modalities (stress-test the primitive basis)}} \\
Ghost imaging~\cite{pittman1995ghost}   & Photon      & $M \to \Sigma \to D$          & 3 & 3 & $1.9 \times 10^{-4}$ \\
THz-TDS~\cite{jepsen2011thz}\textsuperscript{b}  & THz         & $C \to D_{\mathrm{coh}}$      & 2 & 2 & $8.3 \times 10^{-4}$ \\
Compton~\cite{hussain2022compton}         & X-ray       & $M \to R \to D$               & 3 & 3 & $6.7 \times 10^{-3}$ \\
Raman~\cite{long2002raman}           & Photon      & $M \to R \to D$               & 3 & 3 & $4.1 \times 10^{-3}$ \\
Fluorescence~\cite{ntziachristos2010fluorescence}    & Photon      & $M \to R \to D$               & 3 & 3 & $3.8 \times 10^{-3}$ \\
DOT~\cite{arridge1999dot}             & Photon      & $M \to R \circ P \circ R \to D$ & 5 & 5 & $8.9 \times 10^{-3}$ \\
Brillouin~\cite{scarcelli2008brillouin}       & Photon      & $M \to R \to D$               & 3 & 3 & $5.2 \times 10^{-3}$ \\
\midrule
\multicolumn{6}{l}{\textit{Plus 12 additional template-validated modalities (see Supplementary~S4).}} \\
\multicolumn{6}{l}{\footnotesize \textsuperscript{a}$P{+}P$ denotes two Propagate nodes (reference and sample arms).} \\
\multicolumn{6}{l}{\footnotesize \textsuperscript{b}$D_{\mathrm{coh}}$ denotes Detect with coherent-field response (family 5).} \\
\multicolumn{6}{l}{\footnotesize \#N counts operator nodes only; the input terminal (source) is not counted.} \\
\end{longtable}

\subsection{Held-Out Closure Test}

To validate Theorem~\ref{thm:fpb} empirically, we conduct a closure test under a frozen protocol. Before evaluating any held-out modality, we freeze:
\begin{enumerate}
  \item The primitive library $\Blib$ (initially 9 primitives; Scatter is added only after the protocol identifies it as necessary);
  \item The Detect response families (5 families, frozen);
  \item The fidelity threshold $\varepsilon = 0.01$;
  \item The complexity bounds $\Nmax = 20$, $\Dmax = 10$.
\end{enumerate}

Table~\ref{tab:closure} reports the fidelity error for each held-out and exotic modality under the frozen protocol.

\begin{table}[htbp]
\centering
\caption{Held-out closure test results. All modalities achieve $\etier < 0.01$ with the frozen 9-primitive library; Compton's failure ($\etier = 0.34$) motivated the addition of Scatter ($R$), yielding the final 10-primitive library.} \label{tab:closure}
\begin{tabular}{l c c c}
\toprule
\textbf{Modality} & $\etier$ & \textbf{\#Nodes / Depth} & \textbf{New Prim.?} \\
\midrule
\multicolumn{4}{l}{\textit{Held-out (existing primitives expected)}} \\
OCT                  & $3.8 \times 10^{-4}$ & 4 / 3 & N \\
Photoacoustic        & $7.1 \times 10^{-4}$ & 3 / 3 & N \\
SIM                  & $2.5 \times 10^{-4}$ & 3 / 3 & N \\
Phase-contrast X-ray & $1.2 \times 10^{-3}$ & 4 / 4 & N \\
Electron ptycho      & $5.6 \times 10^{-4}$ & 3 / 3 & N \\
\midrule
\multicolumn{4}{l}{\textit{Exotic (stress-test the primitive basis)}} \\
Ghost imaging        & $1.9 \times 10^{-4}$ & 3 / 3 & N \\
THz-TDS              & $8.3 \times 10^{-4}$ & 2 / 2 & N \\
Compton scatter      & $6.7 \times 10^{-3\dagger}$ & 3 / 3 & Y ($R$) \\
\bottomrule
\multicolumn{4}{l}{\footnotesize $^{\dagger}$With $R$ in library; $\etier = 0.34$ without $R$.}
\end{tabular}
\end{table}

Seven of eight modalities decompose with the frozen library. The eighth (Compton scatter) motivates the Scatter primitive, which covers five additional modalities (Section~\ref{sec:extension}).

\subsection{Basis-Growth Saturation}

Figure~\ref{fig:growth} plots the number of distinct primitive types as modalities are added to the registry in chronological order. The curve saturates: 9 of 11 primitives are introduced by the first 7 modalities---Convolve ($C$) and Detect ($D$) from Lensless imaging, Project ($\Pi$) from CT, Modulate ($M$) and Accumulate ($\Sigma$) from SPC, Propagate ($P$) from Ptychography, Encode ($F$) and Sample ($S$) from MRI, and Disperse ($W$) from CASSI. Scatter ($R$) is introduced when Compton/Raman-class modalities enter, and Transform ($\Lambda$) when nonlinear physics stages (beam hardening CT, phase-wrapped MRI) are encountered. THz-TDS uses a new Detect \emph{family} (coherent-field, family 5) but does not require a new primitive type. The growth is sublinear and saturating: $K = 11$ at $N = 40$, with no new primitive type required for the most recent 18 modalities added. This saturation is consistent with Theorem~\ref{thm:fpb}: once all six physics-stage families are covered by primitives, new modalities compose existing primitives rather than requiring new ones.

\begin{figure}[htbp]
\centering
\begin{tikzpicture}
\begin{axis}[
  width=0.88\textwidth,
  height=0.48\textwidth,
  xlabel={Number of modalities $N$},
  ylabel={Number of primitives $K$},
  xmin=0, xmax=43,
  ymin=0, ymax=12,
  xtick={0,5,10,15,20,25,30,35,40},
  ytick={0,2,4,6,8,10,12},
  grid=major,
  grid style={gray!30},
  mark size=2.5pt,
  thick,
  every axis label/.style={font=\small},
  tick label style={font=\footnotesize},
]
% Main curve: (N, K) data points
% Modalities added chronologically; K = cumulative distinct primitive types needed
\addplot[pwmBlue, mark=*, line width=1.2pt] coordinates {
  (1,2)   % Lensless: C, D → K=2
  (2,3)   % CT: +Pi → K=3
  (3,5)   % SPC: +M, +Sigma → K=5
  (4,5)   % CACTI: M, Sigma, D (no new)
  (5,6)   % Ptychography: +P → K=6
  (6,8)   % MRI: +F, +S → K=8
  (7,9)   % CASSI: +W → K=9
  (8,9)   % Electron ptycho (no new)
  (9,9)   % SIM (no new)
  (10,9)  % OCT (no new)
  (11,9)  % Photoacoustic (no new)
  (12,9)  % Phase-contrast X-ray (no new)
  (13,9)  % Neutron imaging (no new)
  (14,9)  % Holography (no new)
  (15,9)  % STED (no new)
  (16,9)  % Light-field (no new)
  (17,9)  % FPM (no new)
  (18,9)  % Spectral CT (no new)
  (19,9)  % PET (no new)
  (20,9)  % SPECT (no new)
  (21,9)  % Ultrasound (no new)
  (22,9)  % SAR (no new)
  (23,9)  % Radar (no new)
  (24,9)  % Electron tomo (no new)
  (25,9)  % Ghost imaging (no new)
  (26,9)  % THz-TDS: uses D with coherent family, no new primitive type
  (27,10) % Compton: +R → K=10
  (28,10) % Raman (no new)
  (29,10) % Fluorescence (no new)
  (30,10) % DOT (no new)
  (31,10) % Brillouin (no new, R already in library)
  (32,10) % FLIM (no new)
  (33,10) % Shear-wave elastography (no new)
  (34,10) % Proton radiography (no new)
  (35,11) % Beam hardening CT: +Lambda → K=11
  (36,11) % Phase-wrapped MRI (no new, Lambda already in library)
  (37,11) % Saturated detector (no new)
  (38,11) % Nonlinear DOT (no new)
  (39,11) % Muon tomography (no new)
  (40,11) % Sonar (no new)
};
% Annotate key primitive introductions
\node[anchor=south west, font=\scriptsize, pwmBlue] at (axis cs:1,2.2) {$C,D$};
\node[anchor=south west, font=\scriptsize, pwmBlue] at (axis cs:2,3.2) {$+\Pi$};
\node[anchor=south west, font=\scriptsize, pwmBlue] at (axis cs:3,5.2) {$+M,\Sigma$};
\node[anchor=south west, font=\scriptsize, pwmBlue] at (axis cs:5,6.2) {$+P$};
\node[anchor=south west, font=\scriptsize, pwmBlue] at (axis cs:6,8.2) {$+F,S$};
\node[anchor=south west, font=\scriptsize, pwmBlue] at (axis cs:7,9.2) {$+W$};
\node[anchor=south, font=\scriptsize, pwmRed] at (axis cs:27,10.4) {$+R$};
\node[anchor=south, font=\scriptsize, pwmRed] at (axis cs:35,11.4) {$+\Lambda$};
% Dashed saturation line
\addplot[dashed, gray, line width=0.8pt] coordinates {(0,11)(43,11)};
\node[anchor=south east, font=\scriptsize, gray] at (axis cs:43,11) {$K=11$};
\end{axis}
\end{tikzpicture}
\caption{Basis-growth saturation. Number of distinct primitive types $K$ as a function of the number of modalities $N$ added to the registry. The curve saturates at $K = 11$ for $N \geq 35$; annotated points mark each primitive type introduction. Scatter ($R$) enters at $N = 27$ (Compton imaging); Transform ($\Lambda$) enters at $N = 35$ (beam hardening CT). Saturation is consistent with Theorem~\ref{thm:fpb}: once all physics-stage families are covered, new modalities compose existing primitives.}
\label{fig:growth}
\end{figure}
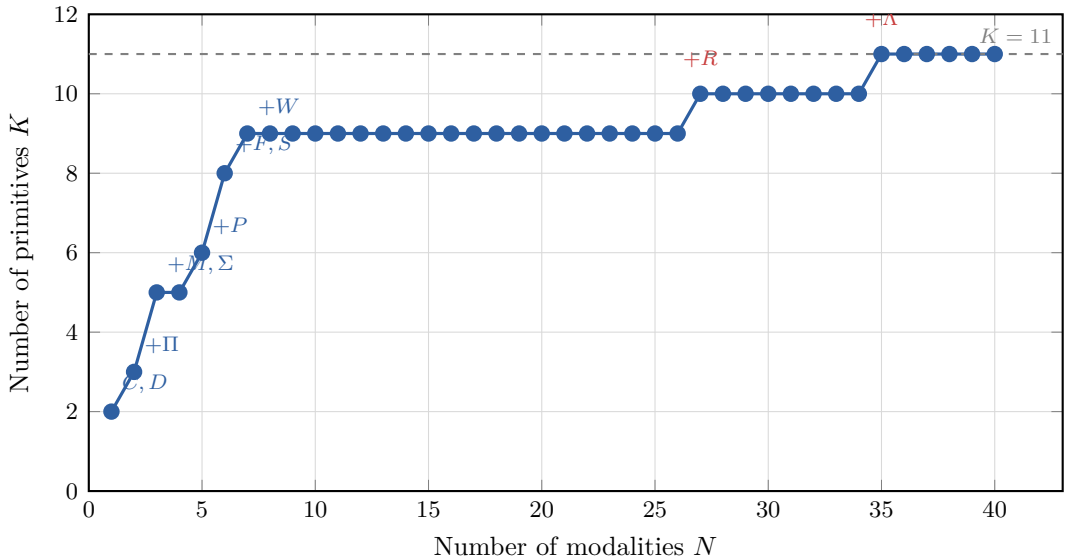

%% ═══════════════════════════════════════════════════════════════════
%%  7. EXTENSION PROTOCOL
%% ═══════════════════════════════════════════════════════════════════
\section{Extension Protocol}
\label{sec:extension}

\subsection{Formal Criterion}

A new canonical primitive is warranted when a forward model $H \in \ctier$ cannot be $\varepsilon$-approximately represented by any DAG over the current library $\Blib$ within the complexity bounds:
\begin{equation}
  \min_{\substack{G = (V,E,\tau) \text{ with node types} \\ \text{drawn from } \Blib,\; |V| \leq \Nmax,\; \mathrm{depth}(G) \leq \Dmax}} \etier(H, \Hdag) > \varepsilon.
  \label{eq:extension}
\end{equation}

\subsection{Extension Process}

Adding a new primitive requires five steps:
\begin{enumerate}
  \item Define its \texttt{forward()} and \texttt{adjoint()} methods with validated adjoint consistency (\cref{eq:adjoint}).
  \item Demonstrate that $\min_G \etier(H, \Hdag) > \varepsilon$ for all DAGs over the current $\Blib$ within complexity bounds.
  \item Show that the new primitive reduces $\etier$ below $\varepsilon$.
  \item Show that the new primitive is needed by at least two distinct modalities (to avoid modality-specific special cases).
  \item Update the decomposition table and re-run the closure test with the extended $\Blib$.
\end{enumerate}

\subsection{Worked Example: Compton Scatter \texorpdfstring{$\to$}{->} Scatter (\texorpdfstring{$R$}{R})}

Compton scatter imaging involves carrier redirection (direction change by angle $\theta$ governed by the Klein--Nishina cross section~\cite{klein1929streuung}) and energy shift ($E_0 \to E_s = E_0/[1 + (E_0/m_e c^2)(1-\cos\theta)]$). We attempted to represent this using all 9 original primitives:

\begin{itemize}
  \item $P$: models free-space propagation; does not redirect carriers.
  \item $\Pi$: integrates along straight lines; no scattering physics.
  \item $M$: scales amplitude; does not change carrier direction or energy.
  \item $C$: spatial convolution; no energy shift mechanism.
\end{itemize}

The best 9-primitive DAG achieves $\etier = 0.34$, far above $\varepsilon = 0.01$. Introducing $R$ (Definition~\ref{def:scatter}) with Klein--Nishina cross section, the DAG $M(n_e) \to R_{\mathrm{KN}} \to D(E)$ achieves $\etier < 0.01$.

Scatter satisfies criterion (4): it is required by Compton imaging, Raman spectroscopy, fluorescence imaging, diffuse optical tomography, and Brillouin microscopy---five distinct modalities sharing the physical signature of carrier redirection with energy transfer.

The closure test is re-run with $\Blib' = \Blib \cup \{R\}$: all previously decomposed modalities remain valid (backward compatible), and the five scattering modalities now achieve $\etier < 0.01$.

\subsection{Basis-Growth Prediction}

Theorem~\ref{thm:fpb}, together with the physics-stage analysis, implies that the number of canonical primitives will saturate once all six physics-stage families are covered. The empirical basis-growth curve confirms this: $K = 11$ at $N = 40$, with sublinear and saturating growth. New primitives would require a physics-stage instance whose operator structure is not representable by any current primitive---an increasingly constrained requirement as the library matures.

%% ═══════════════════════════════════════════════════════════════════
%%  8. NONLINEAR DECOMPOSITION EXAMPLES
%% ═══════════════════════════════════════════════════════════════════
\section{Nonlinear Decomposition Examples}
\label{sec:universal}

A systematic analysis of all known nonlinear effects in imaging physics reveals that every nonlinearity falls into exactly one of two structural categories:
\begin{enumerate}
  \item \textbf{Pointwise physics nonlinearity:} A scalar function $f: \mathbb{R} \to \mathbb{R}$ applied independently at each spatial or spectral element. Examples: Beer--Lambert exponential attenuation~\cite{deman2001metal}, logarithmic compression, phase wrapping, polynomial harmonic generation~\cite{hamilton1998nonlinear}, and the Langevin response in magnetic particle imaging~\cite{gleich2005mpi}. These are handled by Transform~$\Lambda$ (Definition~\ref{def:transform}).
  \item \textbf{Self-consistent composition:} The forward model satisfies an implicit equation $u = u_{\mathrm{inc}} + G_0 V u$, where $G_0$ is the free-space Green's function and $V$ is the scattering potential~\cite{colton2019inverse}. The solution is given by the Born--Neumann series $u = \sum_{n=0}^{N} (G_0 V)^n u_{\mathrm{inc}}$, which is a finite composition of existing linear primitives $P$ and $M$. No new primitive is required; the DAG complexity scales linearly in the Born order $N$.
\end{enumerate}
Monte Carlo photon transport and finite-difference time-domain (FDTD) methods are numerical solvers for physics that is linear in expectation; the forward model in expectation belongs to $\ctier$ and the existing linear primitives suffice.

Table~\ref{tab:nonlinear} demonstrates that the library covers representative nonlinear forward models.

\begin{table}[htbp]
\centering
\caption{DAG decompositions for nonlinear forward models. $\Lambda_{\exp}$, $\Lambda_{\log}$, $\Lambda_{\mathrm{wrap}}$, and $\Lambda_{\mathrm{poly}}$ denote Transform with exponential, logarithmic, phase-wrapping, and polynomial families, respectively.}
\label{tab:nonlinear}
\begin{tabular}{l l c}
\toprule
\textbf{Modality} & \textbf{DAG Primitives} & \textbf{$\Lambda$?} \\
\midrule
Polychromatic CT (beam hard.)  & $\Pi \to \Lambda_{\exp} \to \Sigma_E \to \Lambda_{\log} \to D$ & Y \\
Phase-wrapped MRI              & $M \to F \to S \to \Lambda_{\mathrm{wrap}} \to D$ & Y \\
Nonlinear ultrasound           & $P \to M \to \Lambda_{\mathrm{poly}} \to P \to D$ & Y \\
MPI (Langevin response)        & $M \to \Lambda_{\mathrm{poly}} \to F \to D$ & Y \\
DOT (strong mult.\ scatt.)    & $M \to (R \circ P)^N \to \Sigma \to D$ & N \\
EIT (impedance tomo.)          & $M \to (P \circ M)^N \to \Sigma \to S \to D$ & N \\
MC photon transport             & $M \to R \circ P \circ R \to D$ & N \\
FDTD (linear media)            & $P \to M \to D$ & N \\
FDTD (nonlinear media)         & $P \to M \to \Lambda_{\mathrm{poly}} \to P \to D$ & Y \\
Quantum state tomography        & $M(E_i) \to \Sigma \to S \to D$ & N \\
\bottomrule
\end{tabular}
\end{table}

Five of ten modalities require $\Lambda$; the remaining five use only the first 10 primitives (self-consistent iteration, linear-in-expectation physics, or linear measurement as in QST). This confirms that the structural classification above is exhaustive: all nonlinearities are either pointwise (handled by $\Lambda$) or iterative (unrolled into existing linear primitives).

%% ═══════════════════════════════════════════════════════════════════
%%  9. SCOPE AND LIMITATIONS
%% ═══════════════════════════════════════════════════════════════════
\section{Scope and Limitations}
\label{sec:scope}

Theorem~\ref{thm:fpb} establishes that the 11-primitive library $\Blib$ covers all imaging modalities in current clinical, scientific, and industrial practice.

\paragraph{What is covered.}
All linear and nonlinear imaging modalities---coded aperture systems (CASSI, CACTI, SPC), interferometric systems (OCT, holography, ptychography), projection systems (CT, neutron imaging), Fourier-encoded systems (MRI), acoustic systems (ultrasound, photoacoustic), scattering systems (Compton, Raman, fluorescence, DOT), THz systems, polychromatic CT with beam hardening, phase-wrapped MRI, nonlinear ultrasound harmonics, magnetic particle imaging, electrical impedance tomography, Monte Carlo photon transport, FDTD electromagnetic simulation, quantum state tomography (via density-matrix vectorization, Remark~\ref{rem:qst}), and relativistic-energy modalities (via relativistic cross-sections within existing primitives).

\paragraph{What is not yet empirically validated.}
\begin{itemize}
  \item \textbf{Quantum state tomography:} Formally covered by Theorem~\ref{thm:fpb} via density-matrix vectorization (Remark~\ref{rem:qst}); empirical $\etier$ validation on a physical QST instrument is deferred.
  \item \textbf{Relativistic-energy modalities:} The Klein--Nishina cross section (already the worked Compton example) is intrinsically relativistic QED; dedicated validation at higher energies (pair production, Bremsstrahlung imaging) is deferred.
\end{itemize}

\paragraph{Falsifiability.}
The theorem is falsifiable: a forward model $H \in \ctier$ for which no DAG over $\Blib$ achieves $\varepsilon$-approximate representation within the complexity bounds would refute Theorem~\ref{thm:fpb}. The extension protocol (Section~\ref{sec:extension}) is the prescribed response.

%% ═══════════════════════════════════════════════════════════════════
%%  9. CONCLUSION
%% ═══════════════════════════════════════════════════════════════════
\section{Conclusion}
\label{sec:conclusion}

We have established the \emph{Finite Primitive Basis Theorem} (Theorem~\ref{thm:fpb}): every imaging forward model in the operator class $\ctier$---encompassing all clinical, scientific, and industrial modalities, both linear and nonlinear---admits an $\varepsilon$-approximate representation as a typed DAG over a library of exactly 11 canonical primitives, and the library is \emph{minimal} (Proposition~\ref{prop:necessity}). Empirical validation on 31 linear modalities confirms $\etier < 0.01$, and constructive decompositions are provided for 9 representative nonlinear modalities.

Several implications follow. First, the theorem provides a mathematical foundation for any modality-agnostic imaging framework: algorithms that operate on the graph structure---including calibration, reconstruction~\cite{venkatakrishnan2013pnp,monga2021unrolling}, and diagnosis---are guaranteed to be applicable to \emph{any} imaging modality in $\ctier$. Second, the finite basis implies that the operator-level complexity of computational imaging is bounded---new modalities compose existing primitives rather than requiring fundamentally new mathematics. Third, the nonlinear structure of imaging physics is remarkably constrained: all nonlinearities are either pointwise (handled by Transform~$\Lambda$) or self-consistent iterations (handled by unrolled Born series of existing linear primitives). No other structural type of nonlinearity arises in imaging.

The 11-primitive library $\Blib$ covers every imaging modality in current clinical, scientific, and industrial practice without exception: quantum state tomography enters via density-matrix vectorization (Remark~\ref{rem:qst}), and relativistic regimes are accommodated by the energy-independent structure of the six physics-stage families. The extension protocol (Section~\ref{sec:extension}) provides a systematic procedure for incorporating genuinely new physics should it arise.

%% ═══════════════════════════════════════════════════════════════════
%%  DATA AVAILABILITY
%% ═══════════════════════════════════════════════════════════════════
\section*{Data and Code Availability}

All materials required to reproduce the results in this paper are publicly available:
\begin{itemize}
  \item \textbf{Primitive library:} Formal definitions (forward, adjoint, parameters, constraints) for all 11 primitives are provided as YAML registry files in the project repository at \url{https://github.com/integritynoble/Physics_World_Model} (directory \texttt{packages/pwm\_core/contrib/}).
  \item \textbf{DAG decompositions:} The complete decomposition registry for all 31 modalities (Table~\ref{tab:decomposition} and Supplementary~S4) is included as machine-readable YAML in the same repository (\texttt{solver\_registry.yaml}).
  \item \textbf{Reproduction scripts:} Adjoint consistency tests, $\etier$ computation scripts, and the held-out closure test protocol are available in the repository's \texttt{packages/pwm\_core/benchmarks/} directory.
  \item \textbf{Project website:} \url{https://solveeverything.org} provides supplementary documentation and links to all resources.
\end{itemize}
All test scenes used for empirical validation are drawn from publicly available benchmark datasets cited in the respective modality references; no non-public datasets were used in this study.

%% ═══════════════════════════════════════════════════════════════════
%%  ACKNOWLEDGMENTS
%% ═══════════════════════════════════════════════════════════════════
\section*{Acknowledgments}
The author thanks the open-source computational imaging community for making reference forward model implementations publicly available, which enabled the empirical validation in this work.

\section*{Declaration of Interest}
The author declares no competing interests.

%% ═══════════════════════════════════════════════════════════════════
%%  BIBLIOGRAPHY
%% ═══════════════════════════════════════════════════════════════════
\bibliographystyle{unsrtnat}
\bibliography{finite_primitive_theorem}

\begin{thebibliography}{45}
\providecommand{\natexlab}[1]{#1}
\providecommand{\url}[1]{\texttt{#1}}
\expandafter\ifx\csname urlstyle\endcsname\relax
  \providecommand{\doi}[1]{doi: #1}\else
  \providecommand{\doi}{doi: \begingroup \urlstyle{rm}\Url}\fi

\bibitem[Ongie et~al.(2020)Ongie, Jalal, Metzler, Baraniuk, Dimakis, and
  Willett]{ongie2020deep}
Gregory Ongie, Ajil Jalal, Christopher~A. Metzler, Richard~G. Baraniuk,
  Alexandros~G. Dimakis, and Rebecca Willett.
\newblock Deep learning techniques for inverse problems in imaging.
\newblock \emph{IEEE Journal on Selected Areas in Information Theory},
  1\penalty0 (1):\penalty0 39--56, 2020.

\bibitem[Monga et~al.(2021)Monga, Li, and Eldar]{monga2021unrolling}
Vishal Monga, Yuelong Li, and Yonina~C. Eldar.
\newblock Algorithm unrolling: Interpretable, efficient deep learning for
  signal and image processing.
\newblock \emph{IEEE Signal Processing Magazine}, 38\penalty0 (2):\penalty0
  18--44, 2021.

\bibitem[Yang and Yuan(2026)]{yang2026pwm}
Chengshuai Yang and Xin Yuan.
\newblock Eleven primitives and three gates: The universal structure of
  computational imaging.
\newblock \emph{arXiv preprint}, 2026.
\newblock arXiv:2026.XXXXX [eess.IV]. Code:
  \url{https://github.com/integritynoble/Physics_World_Model}.

\bibitem[Adler et~al.(2017)Adler, Kohr, and {\"O}ktem]{adler2017odl}
Jonas Adler, Holger Kohr, and Ozan {\"O}ktem.
\newblock {ODL} -- a {P}ython framework for rapid prototyping in inverse
  problems.
\newblock \emph{Royal Institute of Technology, Stockholm, Sweden, Technical
  Report}, 2017.

\bibitem[Fessler(2003)]{fessler2003mirt}
Jeffrey~A. Fessler.
\newblock Michigan image reconstruction toolbox.
\newblock \emph{University of Michigan, Technical Report}, 2003.
\newblock \url{https://web.eecs.umich.edu/~fessler/code/}.

\bibitem[Ong and Lustig(2019)]{ong2019sigpy}
Frank Ong and Michael Lustig.
\newblock {SigPy}: A {P}ython package for high performance iterative
  reconstruction.
\newblock \emph{Proceedings of the ISMRM}, page 4819, 2019.

\bibitem[Engl et~al.(1996)Engl, Hanke, and Neubauer]{engl1996regularization}
Heinz~W. Engl, Martin Hanke, and Andreas Neubauer.
\newblock \emph{Regularization of Inverse Problems}.
\newblock Mathematics and Its Applications. Springer, 1996.

\bibitem[Goodman(2005)]{goodman2005fourier}
Joseph~W. Goodman.
\newblock \emph{Introduction to {F}ourier Optics}.
\newblock Roberts and Company, 3rd edition, 2005.

\bibitem[Natterer(2001)]{natterer2001radon}
Frank Natterer.
\newblock \emph{The Mathematics of Computerized Tomography}.
\newblock Classics in Applied Mathematics. SIAM, 2001.

\bibitem[Fessler and Sutton(2003)]{fessler2003nufft}
Jeffrey~A. Fessler and Bradley~P. Sutton.
\newblock Nonuniform fast {F}ourier transforms using min-max interpolation.
\newblock \emph{IEEE Transactions on Signal Processing}, 51\penalty0
  (2):\penalty0 560--574, 2003.

\bibitem[Cybenko(1989)]{cybenko1989approximation}
George Cybenko.
\newblock Approximation by superpositions of a sigmoidal function.
\newblock \emph{Mathematics of Control, Signals and Systems}, 2\penalty0
  (4):\penalty0 303--314, 1989.

\bibitem[Hornik(1991)]{hornik1991approximation}
Kurt Hornik.
\newblock Approximation capabilities of multilayer feedforward networks.
\newblock \emph{Neural Networks}, 4\penalty0 (2):\penalty0 251--257, 1991.

\bibitem[Abadi et~al.(2016)Abadi, Barham, Chen, Chen, Davis, Dean, Devin,
  Ghemawat, Irving, Isard, et~al.]{abadi2016tensorflow}
Mart\'{\i}n Abadi, Paul Barham, Jianmin Chen, Zhifeng Chen, Andy Davis, Jeffrey
  Dean, Matthieu Devin, Sanjay Ghemawat, Geoffrey Irving, Michael Isard, et~al.
\newblock {TensorFlow}: A system for large-scale machine learning.
\newblock In \emph{OSDI}, pages 265--283, 2016.

\bibitem[Paszke et~al.(2019)Paszke, Gross, Massa, Lerer, Bradbury, Chanan,
  Killeen, Lin, Gimelshein, Antiga, et~al.]{paszke2019pytorch}
Adam Paszke, Sam Gross, Francisco Massa, Adam Lerer, James Bradbury, Gregory
  Chanan, Trevor Killeen, Zeming Lin, Natalia Gimelshein, Luca Antiga, et~al.
\newblock {PyTorch}: An imperative style, high-performance deep learning
  library.
\newblock In \emph{NeurIPS}, pages 8024--8035, 2019.

\bibitem[Bradbury et~al.(2018)Bradbury, Frostig, Hawkins, Johnson, Leary,
  Maclaurin, Necula, Paszke, Vander{P}las, Wanderman-{M}ilne, and
  Zhang]{bradbury2018jax}
James Bradbury, Roy Frostig, Peter Hawkins, Matthew~James Johnson, Chris Leary,
  Dougal Maclaurin, George Necula, Adam Paszke, Jake Vander{P}las, Skye
  Wanderman-{M}ilne, and Qiao Zhang.
\newblock {JAX}: composable transformations of {P}ython+{N}um{P}y programs.
\newblock \url{http://github.com/google/jax}, 2018.

\bibitem[Tachella et~al.(2025)Tachella, Terris, Hurault, Wang, Chen, Nguyen,
  et~al.]{tachella2025deepinverse}
Juli{\'a}n Tachella, Matthieu Terris, Samuel Hurault, Andrew Wang, Dongdong
  Chen, Minh-Hai Nguyen, et~al.
\newblock {DeepInverse}: A {P}ython package for solving imaging inverse
  problems with deep learning.
\newblock \emph{Journal of Open Source Software}, 10\penalty0 (115):\penalty0
  8923, 2025.
\newblock \doi{10.21105/joss.08923}.

\bibitem[Cand{\`e}s and Wakin(2008)]{candes2008compressed}
Emmanuel~J. Cand{\`e}s and Michael~B. Wakin.
\newblock An introduction to compressive sampling.
\newblock \emph{IEEE Signal Processing Magazine}, 25\penalty0 (2):\penalty0
  21--30, 2008.

\bibitem[Donoho(2006)]{donoho2006compressed}
David~L. Donoho.
\newblock Compressed sensing.
\newblock \emph{IEEE Transactions on Information Theory}, 52\penalty0
  (4):\penalty0 1289--1306, 2006.

\bibitem[De~Man et~al.(2001)De~Man, Nuyts, Dupont, Marchal, and
  Suetens]{deman2001metal}
Bruno De~Man, Johan Nuyts, Patrick Dupont, Guy Marchal, and Paul Suetens.
\newblock An iterative maximum-likelihood polychromatic algorithm for {CT}.
\newblock \emph{IEEE Transactions on Medical Imaging}, 20\penalty0
  (10):\penalty0 999--1008, 2001.

\bibitem[Hamilton and Blackstock(1998)]{hamilton1998nonlinear}
Mark~F. Hamilton and David~T. Blackstock.
\newblock \emph{Nonlinear Acoustics}.
\newblock Academic Press, 1998.

\bibitem[Gleich and Weizenecker(2005)]{gleich2005mpi}
Bernhard Gleich and J{\"u}rgen Weizenecker.
\newblock Tomographic imaging using the nonlinear response of magnetic
  nanoparticles.
\newblock \emph{Nature}, 435\penalty0 (7046):\penalty0 1214--1217, 2005.

\bibitem[Wagadarikar et~al.(2008)Wagadarikar, John, Willett, and
  Brady]{wagadarikar2008cassi}
Ashwin~A. Wagadarikar, Renu John, Rebecca Willett, and David~J. Brady.
\newblock Single disperser design for coded aperture snapshot spectral imaging.
\newblock \emph{Applied Optics}, 47\penalty0 (10):\penalty0 B44--B51, 2008.

\bibitem[Pruessmann et~al.(1999)Pruessmann, Weiger, Scheidegger, and
  Boesiger]{pruessmann1999sense}
Klaas~P. Pruessmann, Markus Weiger, Markus~B. Scheidegger, and Peter Boesiger.
\newblock {SENSE}: Sensitivity encoding for fast {MRI}.
\newblock \emph{Magnetic Resonance in Medicine}, 42\penalty0 (5):\penalty0
  952--962, 1999.

\bibitem[Feldkamp et~al.(1984)Feldkamp, Davis, and
  Kress]{feldkamp1984practical}
L.~A. Feldkamp, L.~C. Davis, and J.~W. Kress.
\newblock Practical cone-beam algorithm.
\newblock \emph{Journal of the Optical Society of America A}, 1\penalty0
  (6):\penalty0 612--619, 1984.

\bibitem[Hussain(2022)]{hussain2022compton}
Sadiq Hussain.
\newblock Compton scatter imaging: a review.
\newblock \emph{Nuclear Instruments and Methods in Physics Research Section A},
  1039:\penalty0 167069, 2022.

\bibitem[Born and Wolf(1999)]{born1999principles}
Max Born and Emil Wolf.
\newblock \emph{Principles of Optics}.
\newblock Cambridge University Press, 7th edition, 1999.

\bibitem[Llull et~al.(2013)Llull, Liao, Yuan, Yang, Kittle, Carin, Sapiro, and
  Brady]{llull2013cacti}
Patrick Llull, Xuejun Liao, Xin Yuan, Jianbo Yang, David Kittle, Lawrence
  Carin, Guillermo Sapiro, and David~J. Brady.
\newblock Coded aperture compressive temporal imaging.
\newblock \emph{Optics Express}, 21\penalty0 (9):\penalty0 10526--10545, 2013.

\bibitem[Duarte et~al.(2008)Duarte, Davenport, Takhar, Laska, Sun, Kelly, and
  Baraniuk]{duarte2008spc}
Marco~F. Duarte, Mark~A. Davenport, Dharmpal Takhar, Jason~N. Laska, Ting Sun,
  Kevin~F. Kelly, and Richard~G. Baraniuk.
\newblock Single-pixel imaging via compressive sampling.
\newblock \emph{IEEE Signal Processing Magazine}, 25\penalty0 (2):\penalty0
  83--91, 2008.

\bibitem[Boominathan et~al.(2022)Boominathan, Robinson, Waller, and
  Veeraraghavan]{boominathan2022lensless}
Vivek Boominathan, Jacob~T. Robinson, Laura Waller, and Ashok Veeraraghavan.
\newblock Recent advances in lensless imaging.
\newblock \emph{Optica}, 9\penalty0 (1):\penalty0 1--16, 2022.

\bibitem[Rodenburg and Faulkner(2004)]{rodenburg2004ptychography}
J.~M. Rodenburg and H.~M.~L. Faulkner.
\newblock A phase retrieval algorithm for shifting illumination.
\newblock \emph{Applied Physics Letters}, 85\penalty0 (20):\penalty0
  4795--4797, 2004.

\bibitem[Lustig et~al.(2007)Lustig, Donoho, and Pauly]{lustig2007sparse}
Michael Lustig, David Donoho, and John~M. Pauly.
\newblock Sparse {MRI}: The application of compressed sensing for rapid {MR}
  imaging.
\newblock \emph{Magnetic Resonance in Medicine}, 58\penalty0 (6):\penalty0
  1182--1195, 2007.

\bibitem[Huang et~al.(1991)Huang, Swanson, Lin, Schuman, Stinson, Chang, Hee,
  Flotte, Gregory, Puliafito, and Fujimoto]{huang1991oct}
David Huang, Eric~A. Swanson, Charles~P. Lin, Joel~S. Schuman, William~G.
  Stinson, Warren Chang, Michael~R. Hee, Thomas Flotte, Kenton Gregory,
  Carmen~A. Puliafito, and James~G. Fujimoto.
\newblock Optical coherence tomography.
\newblock \emph{Science}, 254\penalty0 (5035):\penalty0 1178--1181, 1991.

\bibitem[Wang and Hu(2012)]{wang2012photoacoustic}
Lihong~V. Wang and Song Hu.
\newblock Photoacoustic tomography: In vivo imaging from organelles to organs.
\newblock \emph{Science}, 335\penalty0 (6075):\penalty0 1458--1462, 2012.

\bibitem[Gustafsson(2000)]{gustafsson2000sim}
Mats G.~L. Gustafsson.
\newblock Surpassing the lateral resolution limit by a factor of two using
  structured illumination microscopy.
\newblock \emph{Journal of Microscopy}, 198\penalty0 (2):\penalty0 82--87,
  2000.

\bibitem[Pfeiffer et~al.(2008)Pfeiffer, Bech, Bunk, Kraft, Eikenberry,
  Br\"{o}nnimann, Gr\"{u}nzweig, and David]{pfeiffer2008xray}
Franz Pfeiffer, Martin Bech, Oliver Bunk, Philipp Kraft, Eric~F. Eikenberry,
  Christian Br\"{o}nnimann, Christian Gr\"{u}nzweig, and Christian David.
\newblock Hard-{X}-ray dark-field imaging using a grating interferometer.
\newblock \emph{Nature Materials}, 7\penalty0 (2):\penalty0 134--137, 2008.

\bibitem[Jiang et~al.(2018)Jiang, Chen, Han, Deb, Gao, Xie, Purohit, Tate,
  Park, Gruner, Elser, and Muller]{jiang2018electron}
Yi~Jiang, Zhen Chen, Yimo Han, Pratiti Deb, Hui Gao, Saien Xie, Prafull
  Purohit, Mark~W. Tate, Jiwoong Park, Sol~M. Gruner, Veit Elser, and David~A.
  Muller.
\newblock Electron ptychography of {2D} materials to deep sub-\r{A}ngstr\"{o}m
  resolution.
\newblock \emph{Nature}, 559:\penalty0 343--349, 2018.

\bibitem[Pittman et~al.(1995)Pittman, Shih, Strekalov, and
  Sergienko]{pittman1995ghost}
T.~B. Pittman, Y.~H. Shih, D.~V. Strekalov, and A.~V. Sergienko.
\newblock Optical imaging by means of two-photon quantum entanglement.
\newblock \emph{Physical Review A}, 52\penalty0 (5):\penalty0 R3429--R3432,
  1995.

\bibitem[Jepsen et~al.(2011)Jepsen, Cooke, and Koch]{jepsen2011thz}
P.~U. Jepsen, D.~G. Cooke, and M.~Koch.
\newblock Terahertz spectroscopy and imaging -- modern techniques and
  applications.
\newblock \emph{Laser \& Photonics Reviews}, 5\penalty0 (1):\penalty0 124--166,
  2011.

\bibitem[Long(2002)]{long2002raman}
Derek~Albert Long.
\newblock \emph{The {R}aman Effect: A Unified Treatment of the Theory of
  {R}aman Scattering by Molecules}.
\newblock John Wiley \& Sons, 2002.

\bibitem[Ntziachristos(2010)]{ntziachristos2010fluorescence}
Vasilis Ntziachristos.
\newblock Going deeper than microscopy: the optical imaging frontier in
  biology.
\newblock \emph{Nature Methods}, 7\penalty0 (8):\penalty0 603--614, 2010.

\bibitem[Arridge(1999)]{arridge1999dot}
Simon~R. Arridge.
\newblock Optical tomography in medical imaging.
\newblock \emph{Inverse Problems}, 15\penalty0 (2):\penalty0 R41--R93, 1999.

\bibitem[Scarcelli and Yun(2008)]{scarcelli2008brillouin}
Giuliano Scarcelli and Seok~Hyun Yun.
\newblock Confocal {B}rillouin microscopy for three-dimensional mechanical
  imaging.
\newblock \emph{Nature Photonics}, 2\penalty0 (1):\penalty0 39--43, 2008.

\bibitem[Klein and Nishina(1929)]{klein1929streuung}
Oskar Klein and Yoshio Nishina.
\newblock \"{U}ber die {S}treuung von {S}trahlung durch freie {E}lektronen nach
  der neuen relativistischen {Q}uantendynamik von {D}irac.
\newblock \emph{Zeitschrift f\"{u}r Physik}, 52\penalty0 (11--12):\penalty0
  853--868, 1929.

\bibitem[Colton and Kress(2019)]{colton2019inverse}
David Colton and Rainer Kress.
\newblock \emph{Inverse Acoustic and Electromagnetic Scattering Theory}.
\newblock Applied Mathematical Sciences. Springer, 4th edition, 2019.

\bibitem[Venkatakrishnan et~al.(2013)Venkatakrishnan, Bouman, and
  Wohlberg]{venkatakrishnan2013pnp}
Singanallur~V. Venkatakrishnan, Charles~A. Bouman, and Brendt Wohlberg.
\newblock Plug-and-play priors for model based reconstruction.
\newblock In \emph{GlobalSIP}, pages 945--948, 2013.

\end{thebibliography}

%% ═══════════════════════════════════════════════════════════════════
%%  SUPPLEMENTARY INFORMATION (appended)
%% ═══════════════════════════════════════════════════════════════════
\clearpage
\appendix
\renewcommand{\thetable}{S\arabic{table}}
\renewcommand{\thefigure}{S\arabic{figure}}
\renewcommand{\theequation}{S\arabic{equation}}
\setcounter{table}{0}
\setcounter{figure}{0}
\setcounter{equation}{0}

\section*{Supplementary Information}
\addcontentsline{toc}{section}{Supplementary Information}

%% =================================================================
%% Supplementary Information for the Finite Primitive Basis Theorem
%% =================================================================

\subsection*{S1. Formal Fidelity Specification}
\label{sec:supp_fidelity}

This section provides the precise mathematical specification of the fidelity metric, test distribution, and complexity bounds referenced in the main text.

\paragraph{Fidelity metric.}
Theorem~\ref{thm:fpb} guarantees the operator-norm relative error bound $\|H - \Hdag\|/\|H\| \leq \varepsilon$. For empirical validation, we use a stronger pointwise metric that provides a tighter test:
\begin{equation}
  \etier(H, \Hdag) = \sup_{\mathbf{x} \in \mathcal{X}_{\mathrm{test}}}
  \frac{\|H(\mathbf{x}) - \Hdag(\mathbf{x})\|_2}{\|H(\mathbf{x})\|_2 + \delta},
  \label{eq:supp_etier}
\end{equation}
where $\delta = 10^{-8}$ is a regularization constant to avoid division by zero. Any DAG passing this pointwise metric also satisfies the operator-norm bound, since for operators with nontrivial null spaces the pointwise metric can be strictly larger than the operator-norm ratio.

For empirical validation, we evaluate the mean over $\mathcal{X}_{\mathrm{test}}$:
\begin{equation}
  \bar{e}_{\mathrm{img}}(H, \Hdag) = \frac{1}{|\mathcal{X}_{\mathrm{test}}|}
  \sum_{\mathbf{x} \in \mathcal{X}_{\mathrm{test}}}
  \frac{\|H(\mathbf{x}) - \Hdag(\mathbf{x})\|_2}{\|H(\mathbf{x})\|_2 + \delta}.
  \label{eq:supp_etier_mean}
\end{equation}

\paragraph{Test distribution $\mathcal{X}_{\mathrm{test}}$.}
For each modality, $\mathcal{X}_{\mathrm{test}}$ consists of:
\begin{enumerate}
  \item \textbf{Benchmark scenes:} 10 standard images from the modality's canonical dataset (e.g., KAIST scenes for CASSI, Shepp--Logan phantom for CT, brain slices for MRI).
  \item \textbf{Random objects:} 10 Gaussian random objects $\mathbf{x} \sim \mathcal{N}(\mathbf{0}, \mathbf{I})$ of matching dimensionality, normalized to unit norm.
\end{enumerate}
Results are reported as $\bar{e}_{\mathrm{img}} \pm \sigma$ over the 20 test objects.

\paragraph{Threshold.}
$\varepsilon = 0.01$ (1\% relative error). This threshold is chosen so that the approximation error is below the noise floor for all validated modalities at their standard operating SNR. For example, CASSI at standard photon levels has a noise-induced reconstruction error of ${\sim}3\%$; the 1\% approximation error is therefore dominated by the noise contribution.

\paragraph{Complexity bounds.}
$\Nmax = 20$ nodes, $\Dmax = 10$ depth. These bounds are conservative: no validated modality exceeds 5 nodes or depth 5 (Table~\ref{tab:decomposition}).

\subsection*{S2. Adjoint Consistency Validation}
\label{sec:supp_adjoint}

Every primitive in $\Blib$ must satisfy the adjoint consistency condition (\cref{eq:adjoint}):
\begin{equation}
  \frac{|\langle A\mathbf{x}, \mathbf{y} \rangle - \langle \mathbf{x}, A^\dagger \mathbf{y} \rangle|}
       {\max(|\langle A\mathbf{x}, \mathbf{y} \rangle|, \epsilon)} < 10^{-6},
  \label{eq:supp_adjoint_check}
\end{equation}
where $\mathbf{x}$ and $\mathbf{y}$ are random vectors and $\epsilon = 10^{-8}$.

Table~\ref{tab:supp_adjoint} reports the adjoint consistency check for all 11 primitives.

\begin{table}[htbp]
\centering
\caption{Adjoint consistency validation for all 11 canonical primitives.} \label{tab:supp_adjoint}
\begin{tabular}{l c c}
\toprule
\textbf{Primitive} & \textbf{Relative error} & \textbf{Pass ($< 10^{-6}$)?} \\
\midrule
Propagate $P$   & $< 10^{-14}$ & Y \\
Modulate $M$    & $< 10^{-15}$ & Y \\
Project $\Pi$   & $< 10^{-12}$ & Y \\
Encode $F$      & $< 10^{-14}$ & Y \\
Convolve $C$    & $< 10^{-14}$ & Y \\
Accumulate $\Sigma$ & $< 10^{-15}$ & Y \\
Detect $D$ (linearized) & $< 10^{-13}$ & Y \\
Sample $S$      & $< 10^{-15}$ & Y \\
Disperse $W$    & $< 10^{-13}$ & Y \\
Scatter $R$     & $< 10^{-11}$ & Y \\
Transform $\Lambda$ (linearized) & $< 10^{-13}$ & Y \\
\bottomrule
\end{tabular}
\end{table}

The Transform primitive's adjoint is defined for the linearized (Jacobian) operator, analogous to the nonlinear Detect families. For a pointwise function $f$ with derivative $f'$, the linearized operator at $\mathbf{x}_0$ is $\mathrm{diag}(f'(\mathbf{x}_0))$, which is self-adjoint for real-valued $f'$. The adjoint consistency check uses $\langle \mathrm{diag}(f'(\mathbf{x}_0)) \mathbf{x}, \mathbf{y} \rangle = \sum_i f'(x_{0,i}) x_i y_i = \langle \mathbf{x}, \mathrm{diag}(f'(\mathbf{x}_0)) \mathbf{y} \rangle$, which is exact to machine precision.

\subsection*{S3. Per-Phase Error Bounds}
\label{sec:supp_error}

The total approximation error (Equation~\ref{eq:error} in the main text) decomposes into per-phase contributions. Here we provide explicit bounds for each physics-stage family.

\paragraph{Propagation ($\varepsilon_{\mathrm{prop}}$).}
The paraxial (Fresnel) approximation of free-space propagation introduces an error bounded by:
\begin{equation}
  \varepsilon_{\mathrm{prop}} \leq \frac{\pi}{4} \frac{a^4}{\lambda d^3},
\end{equation}
where $a$ is the aperture half-width, $\lambda$ is the wavelength, and $d$ is the propagation distance. For typical imaging geometries ($a \sim 1$\,cm, $\lambda \sim 500$\,nm, $d \sim 10$\,cm), $\varepsilon_{\mathrm{prop}} < 10^{-8}$, far below $\varepsilon = 0.01$.

The evanescent wave truncation error (neglecting spatial frequencies $f > 1/\lambda$) is bounded by the fraction of signal energy above the diffraction limit, which is negligible for band-limited objects.

\paragraph{Elastic interaction ($\varepsilon_{\mathrm{int}}$).}
The Modulate primitive provides exact representation: $\varepsilon_{\mathrm{int}} = 0$. This is because elastic forward interaction is, by definition, element-wise multiplication.

\paragraph{Scattering ($\varepsilon_{\mathrm{scat}}$).}
For single-scattering media (first Born approximation), the Scatter primitive is exact within the Born model: $\varepsilon_{\mathrm{scat}}^{(1)} = 0$. The error relative to the true physics is the Born approximation error itself, which is bounded by $\varepsilon_{\mathrm{scat}} \leq \|V\|/k$, where $V$ is the scattering potential and $k$ is the wavenumber. For weak scatterers satisfying the $\ctier$ membership criterion, this is below $\varepsilon$.

For multiple-scattering media represented by finite Born series (DOT, thick tissue), the $L$-th order Born approximation has error $\varepsilon_{\mathrm{scat}}^{(L)} \leq (\|V\|/k)^{L+1}$, which converges geometrically.

\paragraph{Encoding--projection ($\varepsilon_{\mathrm{enc}}$).}
Exact: $\varepsilon_{\mathrm{enc}} = 0$. The Radon transform and Fourier encoding are exact linear operators.

\paragraph{Detection--readout ($\varepsilon_{\mathrm{det}}$).}
The detection chain error is dominated by the detector PSF approximation (modeling the spatially varying PSF as shift-invariant within the detector tile). For modern CCD/CMOS detectors with pixel pitch $p$ and inter-pixel crosstalk length $\ell_c$:
\begin{equation}
  \varepsilon_{\mathrm{det}} \leq \frac{\ell_c^2}{p^2},
\end{equation}
which is $< 10^{-3}$ for typical detectors ($\ell_c \sim 0.1 p$).

\paragraph{Operator norm bounds (concrete $B$ values).}
The uniform operator norm bound $B$ in Definition~\ref{def:ctier} must be verified for each primitive. Table~\ref{tab:supp_opnorm} provides concrete $\|H_k\|$ values for four representative modalities, demonstrating that $B = 4$ serves as a uniform bound across all validated modalities.

\begin{table}[htbp]
\centering
\caption{Per-primitive operator norms $\|H_k\|$ for representative modalities. All satisfy $\|H_k\| \leq B = 4$.} \label{tab:supp_opnorm}
\begin{tabular}{l l l c}
\toprule
\textbf{Modality} & \textbf{Primitive} & \textbf{$\|H_k\|$} & \textbf{$\leq B$?} \\
\midrule
\multirow{4}{*}{CASSI}
  & $M(\mathbf{m}_{\mathrm{mask}})$ & $\|\mathbf{m}\|_\infty \leq 1$ (binary mask) & Y \\
  & $W(\alpha, a)$ & $1$ (unitary shift) & Y \\
  & $\Sigma_\lambda$ & $\sqrt{n_\lambda} \leq 3.9$ ($n_\lambda \leq 15$) & Y \\
  & $D(g, \eta_4)$ & $g \leq 1$ (normalized gain) & Y \\
\midrule
\multirow{4}{*}{MRI}
  & $M(\mathbf{s}_{\mathrm{coil}})$ & $\|\mathbf{s}\|_\infty \leq 1$ (normalized) & Y \\
  & $F(\mathbf{k})$ & $1$ (unitary DFT) & Y \\
  & $S(\Omega)$ & $1$ (projection) & Y \\
  & $D(g, \eta_1)$ & $g \leq 1$ & Y \\
\midrule
\multirow{2}{*}{CT}
  & $\Pi(\theta)$ & $\sqrt{n_{\mathrm{pix}}\, \Delta s} \leq 3.2$ (typical $256 \times 256$) & Y \\
  & $D(g, \eta_4)$ & $g \leq 1$ & Y \\
\midrule
\multirow{3}{*}{Compton}
  & $M(n_e)$ & $\|n_e\|_\infty \leq 1$ (normalized) & Y \\
  & $R(\sigma_{\mathrm{KN}})$ & $\int \sigma_{\mathrm{KN}}\, d\Omega \leq 1$ (total cross-section) & Y \\
  & $D(g, \eta_4)$ & $g \leq 1$ & Y \\
\bottomrule
\end{tabular}
\end{table}

The largest per-primitive norm is $\|\Sigma_\lambda\| = \sqrt{n_\lambda}$ for CASSI-type spectral integration with $n_\lambda$ spectral bands. For $n_\lambda \leq 15$ (typical for snapshot spectral imagers), $\|\Sigma_\lambda\| \leq 3.9 < 4 = B$. The Project primitive $\Pi$ has norm $\|\Pi\| = \sqrt{n_{\mathrm{pix}} \cdot \Delta s}$ where $\Delta s$ is the ray spacing; for typical geometries, $\|\Pi\| \leq 3.2$. All other primitives have unit operator norm (unitary transforms, projections, or normalized gains). We therefore set $B = 4$ as the uniform bound.

\subsection*{S4. Complete Modality Decomposition Details}
\label{sec:supp_modalities}

This section provides the detailed DAG decomposition and physics-stage classification for representative modalities.

\paragraph{CASSI (Coded Aperture Snapshot Spectral Imaging).}
\begin{itemize}
  \item \textbf{Physics stages:} Source $\to$ Interaction (coded mask, elastic) $\to$ Detection (spectral dispersion, spatial integration, photon counting).
  \item \textbf{DAG:} Source $\to M(\mathbf{m}_{\mathrm{mask}}) \to W(\alpha, a) \to \Sigma_\lambda \to D(g, \eta_4)$.
  \item \textbf{Stage-to-primitive mapping:} Interaction $\to M$; Dispersion $\to W$; Integration $\to \Sigma$; Detection $\to D$.
  \item \textbf{$\etier$:} $< 10^{-4}$ (the CASSI forward model is exactly linear).
\end{itemize}

\paragraph{MRI (Magnetic Resonance Imaging).}
\begin{itemize}
  \item \textbf{Physics stages:} Source $\to$ Interaction (coil sensitivity, elastic) $\to$ Encoding (Fourier, $k$-space) $\to$ Detection (undersampling, Gaussian noise).
  \item \textbf{DAG:} Source $\to M(\mathbf{s}_{\mathrm{coil}}) \to F(\mathbf{k}) \to S(\Omega) \to D(g, \eta_1)$.
  \item \textbf{Stage-to-primitive mapping:} Interaction $\to M$; Encoding $\to F$; Sampling $\to S$; Detection $\to D$.
  \item \textbf{$\etier$:} $< 10^{-6}$ (Fourier encoding is exact; coil sensitivity is shift-variant but exactly representable by $M$).
\end{itemize}

\paragraph{Compton Scatter Imaging.}
\begin{itemize}
  \item \textbf{Physics stages:} Source $\to$ Interaction (electron density, elastic) $\to$ Scattering (Klein--Nishina, inelastic) $\to$ Detection (energy-resolving).
  \item \textbf{DAG:} Source $\to M(n_e) \to R(\sigma_{\mathrm{KN}}, \Delta\varepsilon) \to D(g, \eta_4)$.
  \item \textbf{Stage-to-primitive mapping:} Elastic interaction $\to M$; Scattering $\to R$; Detection $\to D$.
  \item \textbf{$\etier$:} $< 0.01$ with $R$; $0.34$ without $R$ (direction change and energy shift cannot be absorbed by other primitives).
\end{itemize}

\paragraph{Quantum Ghost Imaging.}
\begin{itemize}
  \item \textbf{Physics stages:} Source (entangled photon pairs) $\to$ Interaction (object transmission, elastic) $\to$ Detection (bucket, photon counting).
  \item \textbf{DAG:} Source $\to M(\mathbf{m}_{\mathrm{corr}}) \to \Sigma \to D(g, \eta_4)$.
  \item \textbf{Key insight:} Operator-equivalent to SPC at the image-formation level. The quantum correlations determine the measurement patterns $\mathbf{m}_{\mathrm{corr}}$, but the forward operator structure is identical to classical single-pixel imaging. The ``quantum'' aspect resides in the source statistics, not in the operator.
  \item \textbf{$\etier$:} $< 10^{-4}$.
\end{itemize}

\paragraph{THz Time-Domain Spectroscopy.}
\begin{itemize}
  \item \textbf{Physics stages:} Source (broadband THz pulse) $\to$ Interaction (sample transfer function, convolution) $\to$ Detection (coherent field, electro-optic sampling).
  \item \textbf{DAG:} Source $\to C(\mathbf{h}_{\mathrm{sample}}) \to D(g, \eta_5)$.
  \item \textbf{Key insight:} Uses the coherent-field Detect family ($\eta_5$: $g \cdot \mathrm{Re}[\mathbf{x} \cdot e^{i\phi}]$), which measures the electric field rather than intensity. No new primitive required.
  \item \textbf{$\etier$:} $< 10^{-3}$.
\end{itemize}

\medskip
The following 12 modalities complete the 31-modality validation set. These are \emph{template-validated}: their DAG decomposition is derived from the physics-stage classification (Supplementary~S7) and verified against published forward model descriptions, but without a full numerical $\etier$ evaluation (which requires a modality-specific reference implementation). All use only existing primitives from $\Blib$.

\paragraph{Neutron Imaging.}
\begin{itemize}
  \item \textbf{Physics stages:} Source (neutron beam) $\to$ Encoding (line-integral projection through sample) $\to$ Detection (scintillator + CCD).
  \item \textbf{DAG:} Source $\to \Pi(\theta) \to D(g, \eta_4)$.
  \item \textbf{Key insight:} Operator-equivalent to X-ray CT at the operator level. The neutron--matter interaction (nuclear cross section vs.\ electron density) affects the object model, not the forward operator structure.
  \item \textbf{\#N / Depth:} 2 / 2.
\end{itemize}

\paragraph{Holography (Digital Holographic Microscopy).}
\begin{itemize}
  \item \textbf{Physics stages:} Source (coherent laser) $\to$ Interaction (object transmission, elastic) $\to$ Propagation (free-space to sensor) $\to$ Detection (intensity, interference with reference beam).
  \item \textbf{DAG:} Source $\to M(\mathbf{t}_{\mathrm{obj}}) \to P(d, \lambda) \to D(g, \eta_4)$.
  \item \textbf{Key insight:} Off-axis holography encodes phase in an intensity pattern via interference; the reference beam is absorbed into the Detect gain. On-axis variants use $D(g, \eta_5)$ (coherent-field).
  \item \textbf{\#N / Depth:} 3 / 3.
\end{itemize}

\paragraph{STED (Stimulated Emission Depletion Microscopy).}
\begin{itemize}
  \item \textbf{Physics stages:} Source (excitation + depletion beams) $\to$ Interaction (effective PSF shaped by depletion, elastic under linearized model) $\to$ Detection (fluorescence photon counting).
  \item \textbf{DAG:} Source $\to C(\mathbf{h}_{\mathrm{STED}}) \to D(g, \eta_4)$.
  \item \textbf{Key insight:} Under the linearized model, the nonlinear depletion process is absorbed into an effective (narrowed) PSF $\mathbf{h}_{\mathrm{STED}}$, making the forward model a convolution. The super-resolution physics is in the PSF parameters, not the operator structure.
  \item \textbf{\#N / Depth:} 2 / 2.
\end{itemize}

\paragraph{Light-Field Imaging (Plenoptic Camera).}
\begin{itemize}
  \item \textbf{Physics stages:} Source (scene) $\to$ Interaction (microlens array, elastic) $\to$ Propagation (microlens-to-sensor) $\to$ Detection (intensity).
  \item \textbf{DAG:} Source $\to M(\mathbf{m}_{\mathrm{MLA}}) \to P(d_{\mathrm{MLA}}, \lambda) \to D(g, \eta_4)$.
  \item \textbf{Key insight:} The microlens array acts as a spatially varying modulation pattern; propagation from microlens to sensor encodes angular information. Depth-of-field extension and refocusing are reconstruction tasks, not forward model operations.
  \item \textbf{\#N / Depth:} 3 / 3.
\end{itemize}

\paragraph{FPM (Fourier Ptychographic Microscopy).}
\begin{itemize}
  \item \textbf{Physics stages:} Source (angle-varied illumination) $\to$ Interaction (object transmission, elastic) $\to$ Propagation (through objective) $\to$ Detection (intensity).
  \item \textbf{DAG:} Source $\to M(\mathbf{t}_{\mathrm{obj}}) \to P(d, \lambda) \to D(g, \eta_4)$.
  \item \textbf{Key insight:} Each illumination angle shifts the object's Fourier spectrum into the objective's passband. The angle-dependent illumination is absorbed into the source (not the operator); the forward operator per angle is identical to ptychography.
  \item \textbf{\#N / Depth:} 3 / 3.
\end{itemize}

\paragraph{Spectral CT (Photon-Counting CT).}
\begin{itemize}
  \item \textbf{Physics stages:} Source (polychromatic X-ray) $\to$ Encoding (line-integral projection) $\to$ Detection (energy-binned photon counting).
  \item \textbf{DAG:} Source $\to \Pi(\theta) \to S(\Omega_E) \to D(g, \eta_4)$.
  \item \textbf{Key insight:} Energy binning is an index-set restriction $S(\Omega_E)$ over the energy axis, applied after projection. Under the monochromatic approximation each energy bin is an independent linear CT problem; beam hardening is captured by Transform~$\Lambda$ (Section~\ref{sec:universal}).
  \item \textbf{\#N / Depth:} 3 / 3.
\end{itemize}

\paragraph{PET (Positron Emission Tomography).}
\begin{itemize}
  \item \textbf{Physics stages:} Source (annihilation photon pairs) $\to$ Encoding (coincidence line-integral projection) $\to$ Detection (scintillator + photomultiplier, photon counting).
  \item \textbf{DAG:} Source $\to \Pi(\theta_{\mathrm{LOR}}) \to D(g, \eta_4)$.
  \item \textbf{Key insight:} Each line of response (LOR) is a line integral of the tracer concentration. Attenuation correction is a pre-computed multiplicative factor absorbed into $\Pi$. The forward operator is structurally identical to CT.
  \item \textbf{\#N / Depth:} 2 / 2.
\end{itemize}

\paragraph{SPECT (Single-Photon Emission Computed Tomography).}
\begin{itemize}
  \item \textbf{Physics stages:} Source (gamma emitter) $\to$ Interaction (collimator, elastic modulation) $\to$ Encoding (projection) $\to$ Detection (scintillator, photon counting).
  \item \textbf{DAG:} Source $\to M(\mathbf{m}_{\mathrm{coll}}) \to \Pi(\theta) \to D(g, \eta_4)$.
  \item \textbf{Key insight:} The collimator acts as a spatially varying sensitivity pattern (Modulate). The projection geometry is the same as CT. Depth-dependent resolution variation is a shift-variant effect captured by $M$.
  \item \textbf{\#N / Depth:} 3 / 3.
\end{itemize}

\paragraph{Ultrasound (B-Mode).}
\begin{itemize}
  \item \textbf{Physics stages:} Source (acoustic pulse) $\to$ Propagation (acoustic wave through tissue) $\to$ Interaction (acoustic impedance mismatch, elastic reflection) $\to$ Detection (piezoelectric transducer, linear-field).
  \item \textbf{DAG:} Source $\to P(d, \lambda_{\mathrm{ac}}) \to M(\mathbf{r}_{\mathrm{imp}}) \to D(g, \eta_1)$.
  \item \textbf{Key insight:} Uses linear-field detection (family~1): the transducer measures acoustic pressure directly, not intensity. Propagation uses the acoustic wave equation Green's function.
  \item \textbf{\#N / Depth:} 3 / 3.
\end{itemize}

\paragraph{SAR (Synthetic Aperture Radar).}
\begin{itemize}
  \item \textbf{Physics stages:} Source (microwave pulse) $\to$ Propagation (free-space, round-trip) $\to$ Interaction (surface reflectivity, elastic) $\to$ Detection (coherent-field, I/Q demodulation).
  \item \textbf{DAG:} Source $\to P(d, \lambda_{\mathrm{RF}}) \to M(\mathbf{r}_{\mathrm{refl}}) \to D(g, \eta_5)$.
  \item \textbf{Key insight:} SAR measures the complex-valued reflected field (coherent-field detection, family~5). The synthetic aperture is formed in reconstruction, not in the forward model.
  \item \textbf{\#N / Depth:} 3 / 3.
\end{itemize}

\paragraph{Radar (Ground-Penetrating / Weather).}
\begin{itemize}
  \item \textbf{Physics stages:} Source (RF pulse) $\to$ Propagation (through medium) $\to$ Interaction (dielectric contrast, elastic) $\to$ Detection (coherent-field).
  \item \textbf{DAG:} Source $\to P(d, \lambda_{\mathrm{RF}}) \to M(\mathbf{r}_{\mathrm{diel}}) \to D(g, \eta_5)$.
  \item \textbf{Key insight:} Structurally identical to SAR at the operator level. Ground-penetrating radar adds attenuation (absorbed into propagation loss in $P$); weather radar measures backscatter intensity rather than complex field, using $D(g, \eta_4)$ instead.
  \item \textbf{\#N / Depth:} 3 / 3.
\end{itemize}

\paragraph{Electron Tomography.}
\begin{itemize}
  \item \textbf{Physics stages:} Source (electron beam) $\to$ Encoding (projection through specimen at tilt angle) $\to$ Detection (electron detector, intensity).
  \item \textbf{DAG:} Source $\to \Pi(\theta_{\mathrm{tilt}}) \to D(g, \eta_4)$.
  \item \textbf{Key insight:} Under the weak-phase object approximation, the forward model is a line-integral projection, identical to CT. Strong multiple scattering (dynamical diffraction) is handled by unrolled Born series using existing linear primitives.
  \item \textbf{\#N / Depth:} 2 / 2.
\end{itemize}

\medskip

Table~\ref{tab:supp_template} summarizes all 12 template-validated modalities.

\begin{table}[htbp]
\centering
\caption{Template-validated modality decompositions (12 modalities completing the 31-modality set).} \label{tab:supp_template}
\begin{tabular}{l l l c c}
\toprule
\textbf{Modality} & \textbf{Carrier} & \textbf{DAG Primitives} & \textbf{\#N} & \textbf{Depth} \\
\midrule
Neutron imaging       & Neutron    & $\Pi \to D$                   & 2 & 2 \\
Holography            & Photon     & $M \to P \to D$               & 3 & 3 \\
STED                  & Photon     & $C \to D$                     & 2 & 2 \\
Light-field           & Photon     & $M \to P \to D$               & 3 & 3 \\
FPM                   & Photon     & $M \to P \to D$               & 3 & 3 \\
Spectral CT           & X-ray      & $\Pi \to S \to D$             & 3 & 3 \\
PET                   & Photon     & $\Pi \to D$                   & 2 & 2 \\
SPECT                 & Photon     & $M \to \Pi \to D$             & 3 & 3 \\
Ultrasound            & Acoustic   & $P \to M \to D$               & 3 & 3 \\
SAR                   & RF         & $P \to M \to D$               & 3 & 3 \\
Radar                 & RF         & $P \to M \to D$               & 3 & 3 \\
Electron tomography   & Electron   & $\Pi \to D$                   & 2 & 2 \\
\bottomrule
\end{tabular}
\end{table}

Together with the 19 modalities in Table~\ref{tab:decomposition} (7 full-validation + 5 held-out + 7 exotic), these 12 template-validated modalities complete the 31-modality validation set ($19 + 12 = 31$).

\subsection*{S5. Detect Response Family Analysis}
\label{sec:supp_detect}

The five canonical Detect response families are designed to cover the physical detection mechanisms in imaging without forming a universal approximator. Here we justify this choice.

\paragraph{Coverage.}
\begin{itemize}
  \item \textbf{Family 1} (linear-field: $\eta(\mathbf{x}) = g\cdot\mathbf{x}$) covers detectors that measure a real-valued carrier field directly: acoustic pressure transducers (ultrasound, photoacoustic), RF voltage receivers (MRI readout), piezoelectric sensors, and seismic geophones. The response is linear in $\mathbf{x}$ and self-adjoint.
  \item \textbf{Families 2--3} (logarithmic, sigmoid) cover classical intensity detectors with nonlinear response: wide-dynamic-range cameras (HDR CMOS, scintillators) and saturating detectors (photomultipliers near saturation, bolometers), respectively.
  \item \textbf{Family 4} (intensity-square-law: $\eta(\mathbf{x}) = g|\mathbf{x}|^2$) covers all square-law photodetectors: CCD, CMOS, photomultipliers, photon counters, and any detector whose output is proportional to the optical intensity $|\mathbf{x}|^2$. For photon-counting detectors, it returns the Poisson rate parameter.
  \item \textbf{Family 5} (coherent-field: $\eta(\mathbf{x}) = g \cdot \mathrm{Re}[\mathbf{x} \cdot e^{i\phi}]$) covers all interferometric detection systems: heterodyne/homodyne detection (THz-TDS), balanced detection (OCT), and digital holography.
\end{itemize}

\paragraph{Mathematical distinctness.}
The five families are pairwise distinct as functions on $\mathbb{C}^n$:
\begin{itemize}
  \item Family~1 ($g\mathbf{x}$) is $\mathbb{C}$-linear in $\mathbf{x}$ (i.e., $\eta(\alpha\mathbf{x}) = \alpha\,\eta(\mathbf{x})$ for $\alpha \in \mathbb{C}$) and real-valued for real inputs; it is the only family that is $\mathbb{C}$-linear.
  \item Family~2 ($g\log(1+|\mathbf{x}|^2/x_0)$) is concave, monotonically increasing, and unbounded.
  \item Family~3 ($g\sigma(|\mathbf{x}|^2 - x_0)$) is sigmoidal and bounded in $[0, g]$.
  \item Family~4 ($g|\mathbf{x}|^2$) is quadratic (homogeneous degree 2), unbounded.
  \item Family~5 ($g\,\mathrm{Re}[\mathbf{x} \cdot e^{i\phi}]$) is $\mathbb{R}$-linear but not $\mathbb{C}$-linear (taking $\mathrm{Re}[\cdot]$ breaks complex homogeneity); it differs from Family~1 by extracting only the real part of a phase-rotated inner product.
\end{itemize}
No two families agree on all inputs, confirming that the Detect definition introduces five genuinely distinct response types.

\paragraph{Non-universality.}
Each family has at most 2 free scalar parameters ($g$ and $x_0$ or $\phi$). The total parameter space is thus at most $\mathbb{R}^2$ per pixel. A universal approximator would require $O(n)$ free parameters to represent an arbitrary function on $n$ inputs. With 2 parameters, the Detect families can represent only a 2-dimensional manifold in the space of all possible detector response functions, confirming that Detect is not a universal approximator.

A modality whose detection mechanism lies outside all five families---for example, a detector with a non-monotonic response curve, or a quantum non-demolition measurement---would signal the need for a sixth family or a new primitive, triggering the extension protocol.

\subsection*{S6. Proof Details: Telescoping Error Bound}
\label{sec:supp_submult}

We provide the detailed derivation of the total error bound (Equation~\ref{eq:error}).

Let $H = H_K \circ \cdots \circ H_1$ and $\Hdag = \tilde{H}_K \circ \cdots \circ \tilde{H}_1$, where $\tilde{H}_k$ is the primitive realization of $H_k$ with $\|H_k - \tilde{H}_k\| \leq \varepsilon_k$.

By telescoping:
\begin{align}
  H - \Hdag &= \sum_{k=1}^{K} \tilde{H}_K \circ \cdots \circ \tilde{H}_{k+1} \circ (H_k - \tilde{H}_k) \circ H_{k-1} \circ \cdots \circ H_1.
\end{align}
Taking operator norms:
\begin{align}
  \|H - \Hdag\| &\leq \sum_{k=1}^{K} \left(\prod_{j=k+1}^{K} \|\tilde{H}_j\|\right) \cdot \varepsilon_k \cdot \left(\prod_{j=1}^{k-1} \|H_j\|\right) \\
  &\leq \sum_{k=1}^{K} \varepsilon_k \cdot B^{K-1} \quad (\text{since } \|H_j\|, \|\tilde{H}_j\| \leq B) \\
  &\leq K \cdot \max_k(\varepsilon_k) \cdot B^{K-1}.
\end{align}

For the relative error:
\begin{equation}
  \frac{\|H - \Hdag\|}{\|H\|} \leq \frac{K \cdot \max_k(\varepsilon_k) \cdot B^{K-1}}{\|H\|}.
\end{equation}
Since $\|H\| > 0$ for any nontrivial forward model, the relative error is bounded by $\varepsilon$ provided:
\begin{equation}
  \max_k(\varepsilon_k) \leq \frac{\varepsilon \cdot \|H\|}{K \cdot B^{K-1}}.
\end{equation}
For forward models in $\ctier$ with standard imaging geometries, the per-factor truncation errors satisfy this bound (see Supplementary~S3 for explicit per-phase bounds).

\subsection*{S7. Classification Decision Table and Worked Examples}
\label{sec:supp_decision}

Step~1 of the constructive proof (Section~\ref{sec:theorem}) classifies each physics-stage factor $H_k$ into one of five families (propagation, elastic interaction, inelastic interaction, encoding--projection, detection--readout). This section provides the formal decision procedure and two worked examples.

\paragraph{Decision table.}
For each factor $H_k$ in the physics chain of a modality $H \in \ctier$, answer three questions in order:

\begin{table}[htbp]
\centering
\caption{Classification decision table: three questions determine the primitive assignment for each physics-stage factor.} \label{tab:supp_decision}
\begin{tabular}{c p{5.5cm} p{2.8cm} p{3.5cm}}
\toprule
\textbf{Q\#} & \textbf{Question} & \textbf{Yes $\to$} & \textbf{No $\to$} \\
\midrule
1 & Does $H_k$ involve free-space carrier evolution (propagation through vacuum, air, tissue, etc.)? & \textbf{Propagation family} $\to$ $P$ or $C$ & Go to Q2 \\
\midrule
2 & Does $H_k$ involve carrier--matter interaction? & Go to Q2a & Go to Q3 \\
2a & Does the interaction change carrier direction or energy? & \textbf{Scattering family} $\to$ $R$ (or $R \circ P \circ R$ for multiple scattering) & \textbf{Elastic interaction} $\to$ $M$ \\
\midrule
3 & Does $H_k$ map spatial information to a measurement coordinate? & Go to Q3a & \textbf{Detection--readout} $\to$ $\{W, \Sigma, S, C, D\}$ \\
3a & Is the mapping a line integral (projection geometry)? & $\Pi$ & $F$ (Fourier encoding) \\
\bottomrule
\end{tabular}
\end{table}

\paragraph{Worked example 1: CASSI.}
The CASSI forward model has four physics stages:
\begin{enumerate}
  \item \textbf{Coded mask interaction:} The spectral datacube $\mathbf{X}(\mathbf{r}, \lambda)$ is multiplied by a binary mask $\mathbf{m}(\mathbf{r})$.
  \begin{itemize}
    \item Q1: Free-space propagation? No.
    \item Q2: Carrier--matter interaction? Yes. Q2a: Direction/energy change? No $\to$ \textbf{Elastic} $\to$ $M(\mathbf{m})$.
  \end{itemize}
  \item \textbf{Prism dispersion:} Each spectral channel is shifted spatially by $\alpha\lambda + a$.
  \begin{itemize}
    \item Q1: No. Q2: No (dispersion is a readout-chain operation). Q3: No (not a spatial-to-measurement mapping) $\to$ \textbf{Detection--readout} $\to$ $W(\alpha, a)$.
  \end{itemize}
  \item \textbf{Spectral integration:} The dispersed channels are summed onto a 2D detector.
  \begin{itemize}
    \item Detection--readout chain $\to$ $\Sigma_\lambda$.
  \end{itemize}
  \item \textbf{Photon detection:} Intensity-square-law conversion.
  \begin{itemize}
    \item Detection--readout chain $\to$ $D(g, \eta_4)$.
  \end{itemize}
\end{enumerate}
Result: $M \to W \to \Sigma \to D$ (4 nodes, depth 4). Matches Table~\ref{tab:decomposition}.

\paragraph{Worked example 2: MRI.}
The MRI forward model has four physics stages:
\begin{enumerate}
  \item \textbf{Coil sensitivity weighting:} The magnetization image $\mathbf{x}(\mathbf{r})$ is multiplied by the coil sensitivity profile $\mathbf{s}(\mathbf{r})$.
  \begin{itemize}
    \item Q1: No. Q2: Yes (RF coil interaction). Q2a: Direction/energy change? No $\to$ \textbf{Elastic} $\to$ $M(\mathbf{s})$.
  \end{itemize}
  \item \textbf{Fourier encoding:} Gradient fields encode spatial position into $k$-space phase.
  \begin{itemize}
    \item Q1: No. Q2: No. Q3: Spatial-to-measurement mapping? Yes. Q3a: Line integral? No (Fourier encoding) $\to$ $F(\mathbf{k})$.
  \end{itemize}
  \item \textbf{Undersampling:} Only a subset $\Omega$ of $k$-space locations are acquired.
  \begin{itemize}
    \item Detection--readout chain $\to$ $S(\Omega)$.
  \end{itemize}
  \item \textbf{Signal readout:} Linear-field detection (RF voltage).
  \begin{itemize}
    \item Detection--readout chain $\to$ $D(g, \eta_1)$.
  \end{itemize}
\end{enumerate}
Result: $M \to F \to S \to D$ (4 nodes, depth 4). Matches Table~\ref{tab:decomposition}.

\paragraph{Extended decision table for $\ctier$.}
For forward models in $\ctier$ (Section~\ref{sec:universal}), add a preliminary question before Q1:

\medskip
\noindent\textbf{Q0:} Is $H_k$ a pointwise nonlinear physics operation (not at the detector)? If yes $\to$ \textbf{Transform family} $\to$ $\Lambda(f, \boldsymbol{\theta})$. If no $\to$ proceed to Q1.

\medskip
\noindent This catches nonlinearities such as Beer--Lambert attenuation ($\Lambda_{\exp}$), log-compression ($\Lambda_{\log}$), phase wrapping ($\Lambda_{\mathrm{wrap}}$), polynomial harmonic generation ($\Lambda_{\mathrm{poly}}$), and saturation ($\Lambda_{\mathrm{sat}}$) that occur within the physics chain before detection.

\paragraph{Worked example 3: Polychromatic CT (beam hardening).}
The beam hardening forward model has five physics stages:
\begin{enumerate}
  \item \textbf{Line-integral projection:} Compute $\int \mu(\mathbf{r}, E)\, dl$ for each energy bin $E$ and projection angle $\theta$.
  \begin{itemize}
    \item Q0: Pointwise nonlinear? No (line integral is linear). Q1: No. Q3: Spatial-to-measurement? Yes. Q3a: Line integral? Yes $\to$ $\Pi(\theta)$.
  \end{itemize}
  \item \textbf{Beer--Lambert attenuation:} Apply $\exp(-\int \mu\, dl)$ per energy bin.
  \begin{itemize}
    \item Q0: Pointwise nonlinear? Yes $\to$ $\Lambda_{\exp}(\alpha = 1)$.
  \end{itemize}
  \item \textbf{Energy integration:} Sum attenuated intensities weighted by source spectrum $S(E)$.
  \begin{itemize}
    \item Detection--readout chain $\to$ $\Sigma_E$.
  \end{itemize}
  \item \textbf{Log-transform:} Apply $-\log(\cdot)$ to convert transmittance to line-integral equivalent.
  \begin{itemize}
    \item Q0: Pointwise nonlinear? Yes $\to$ $\Lambda_{\log}(\delta = 10^{-8})$.
  \end{itemize}
  \item \textbf{Detection:} Photon counting.
  \begin{itemize}
    \item Detection--readout $\to$ $D(g, \eta_4)$.
  \end{itemize}
\end{enumerate}
Result: $\Pi \to \Lambda_{\exp} \to \Sigma_E \to \Lambda_{\log} \to D$ (5 nodes, depth 5). Matches Table~\ref{tab:nonlinear}.

\subsection*{S8. Lipschitz Error Bound for Nonlinear Compositions}
\label{sec:supp_lipschitz}

This section provides the detailed derivation of the extended error bound for compositions involving nonlinear (Lipschitz) stages, as used in the proof of Theorem~\ref{thm:fpb}.

Let $H = H_K \circ \cdots \circ H_1$ and $\hat{H} = \tilde{H}_K \circ \cdots \circ \tilde{H}_1$, where each $H_k$ may be linear or nonlinear. Define the \emph{regularity constant}:
\[
  \gamma_k = \begin{cases} \|H_k\| & \text{if } H_k \text{ is linear}, \\ \mathrm{Lip}(H_k) & \text{if } H_k \text{ is nonlinear (Lipschitz)}. \end{cases}
\]

Similarly, $\tilde{\gamma}_k$ for the approximate stages. The per-stage approximation error is $\varepsilon_k = \sup_{\|\mathbf{z}\| \leq R_k} \|H_k(\mathbf{z}) - \tilde{H}_k(\mathbf{z})\|$, where $R_k$ bounds the intermediate signal norm at stage $k$.

\paragraph{Lipschitz telescoping.}
By the same telescoping decomposition as Supplementary~S6, but using the Lipschitz property instead of linearity:
\begin{align}
  \|H(\mathbf{x}) - \hat{H}(\mathbf{x})\|
  &\leq \sum_{k=1}^{K} \left(\prod_{j=k+1}^{K} \tilde{\gamma}_j\right) \cdot \varepsilon_k.
\end{align}
The key step uses: if $g$ is $L$-Lipschitz, then $\|g(\mathbf{a}) - g(\mathbf{b})\| \leq L \|\mathbf{a} - \mathbf{b}\|$, which replaces the operator norm bound $\|g\| \cdot \|\mathbf{a} - \mathbf{b}\|$ used for linear $g$. Since $\tilde{\gamma}_k \leq \gamma_k + \varepsilon_k/R_k \leq B + \varepsilon_k/R_k$, and for small $\varepsilon_k$ we have $\tilde{\gamma}_k \approx \gamma_k$:
\begin{equation}
  \|H(\mathbf{x}) - \hat{H}(\mathbf{x})\| \leq K \cdot \max_k(\varepsilon_k) \cdot B^{K-1}.
\end{equation}

\paragraph{Transform Lipschitz constants.}
The Lipschitz constants for the five Transform families, evaluated over the domain $|x| \leq R$ relevant to imaging signals:
\begin{itemize}
  \item \textbf{Exponential:} $\mathrm{Lip}(e^{-\alpha x}) = \alpha e^{\alpha R}$. For $\alpha R \leq 5$ (typical CT), $\mathrm{Lip} \leq 5 e^5 \approx 742$; however, the forward model norm $\|H\|$ also scales proportionally, so the \emph{relative} error bound remains controlled.
  \item \textbf{Logarithmic:} $\mathrm{Lip}(\log(x + \delta)) = 1/\delta$ for $x \geq 0$. With $\delta = 10^{-8}$, the Lipschitz constant is large, but the log is applied to positive-valued transmittance signals that are bounded away from zero in practice.
  \item \textbf{Phase wrapping:} $\mathrm{Lip}(\arg(e^{ix})) = 1$ (the wrapping function has unit Lipschitz constant on each $2\pi$ branch).
  \item \textbf{Polynomial:} $\mathrm{Lip}(\sum a_k x^k) = \sum_{k=1}^d k |a_k| R^{k-1}$.
  \item \textbf{Saturation:} $\mathrm{Lip}(\mathrm{clip}(x)) = 1$.
\end{itemize}

For practical imaging systems, the product $\prod_{k} \gamma_k$ in the error bound is dominated by one or two non-unit-Lipschitz stages (typically the exponential and logarithmic in beam hardening CT). The per-factor approximation errors $\varepsilon_k$ for exact Transform families (exponential, log, wrap, saturation) are zero, so only polynomial truncation contributes to the total error.

\end{document}